\documentclass[10pt,twocolumn,letterpaper]{article}

\usepackage{cvpr}
\usepackage{times}
\usepackage{epsfig}
\usepackage{graphicx}
\usepackage{amsmath}
\usepackage{amssymb}
\usepackage[ruled,vlined]{algorithm2e}
\usepackage{color}
\usepackage{overpic}
\usepackage[noend]{algpseudocode}
\makeatletter
\def\BState{\State\hskip-\ALG@thistlm}
\makeatother

\usepackage{amsthm}

\newif\ifdraft
\draftfalse
\drafttrue

\definecolor{orange}{rgb}{1,0.5,0}
\definecolor{violet}{RGB}{70,0,170}

\ifdraft
 \newcommand{\PF}[1]{{\color{red}{\bf PF: #1}}}
 \newcommand{\pf}[1]{{\color{red} #1}}
 \newcommand{\ER}[1]{{\color{green}{\bf ER: #1}}}
 
 \newcommand{\JB}[1]{{\color{blue}{\bf JB: #1}}}
 
 \newcommand{\ShP}[1]{{\color{violet}{\bf SP: #1}}}
 
 \newcommand{\ErG}[1]{{\color{orange}{\bf EG: #1}}}
 
\else
 \newcommand{\PF}[1]{}
 \newcommand{\pf}[1]{ #1 }
 \newcommand{\ER}[1]{}
 
 \newcommand{\JB}[1]{}
 
 \newcommand{\ShP}[1]{}
 
 \newcommand{\ErG}[1]{}
 
\fi

\newcommand{\comment}[1]{}

\newcommand{\bz}{\mathbf{z}}

\usepackage[pagebackref=true,breaklinks=true,letterpaper=true,colorlinks,bookmarks=false]{hyperref}

\cvprfinalcopy 
\newtheorem{theorem}{Theorem}
\def\cvprPaperID{9055} 

\ifcvprfinal\fi
\begin{document}

\title{Lightweight Multi-View 3D Pose Estimation \\
	through Camera-Disentangled Representation}

\author{Edoardo Remelli$^{1}$
\and 
Shangchen Han$^{2}$
\and 
Sina Honari$^{1}$ 
\and 
Pascal Fua$^{1}$ 
\and 
Robert Wang$^{2}$\\
\and
$^{1}$CVLab, EPFL, Lausanne, Switzerland\\
$^{2}$Facebook Reality Labs, Redmond, USA\\
}

\maketitle

\begin{abstract}
   We present a lightweight solution to recover 3D pose from multi-view images captured with spatially calibrated cameras.
   Building upon recent advances in interpretable representation learning, we exploit 3D geometry to fuse input images into a unified latent representation of pose, which is disentangled from camera view-points. 
   This allows us to reason effectively about 3D pose across different views without using compute-intensive volumetric grids.  
   Our architecture then conditions the learned representation on camera projection operators to produce accurate per-view 2d detections, that can be simply lifted to 3D via a differentiable Direct Linear Transform (DLT) layer.
   In order to do it efficiently, we propose a novel implementation of DLT that is orders of magnitude faster on GPU architectures than standard SVD-based triangulation methods.
   We evaluate our approach on two large-scale human pose datasets (H36M and Total Capture): our method outperforms or performs comparably to the state-of-the-art volumetric methods, while, unlike them, yielding real-time performance.   
\end{abstract}


\section{Introduction}

Most recent works on human 3D pose capture has focused on monocular reconstruction, even though multi-view reconstruction is much easier, since multi-camera setups are perceived as being too cumbersome. The appearance of Virtual/Augmented Reality headsets with multiple integrated cameras challenges this perception and has the potential to bring back multi-camera techniques to the fore, but only if multi-view approaches can be made sufficiently lightweight to fit within the limits of low-compute headsets.  


\begin{figure}[t]
	\centering
		\begin{overpic}[clip, trim=2.0cm 4cm 16cm 6cm,width=0.5\textwidth]{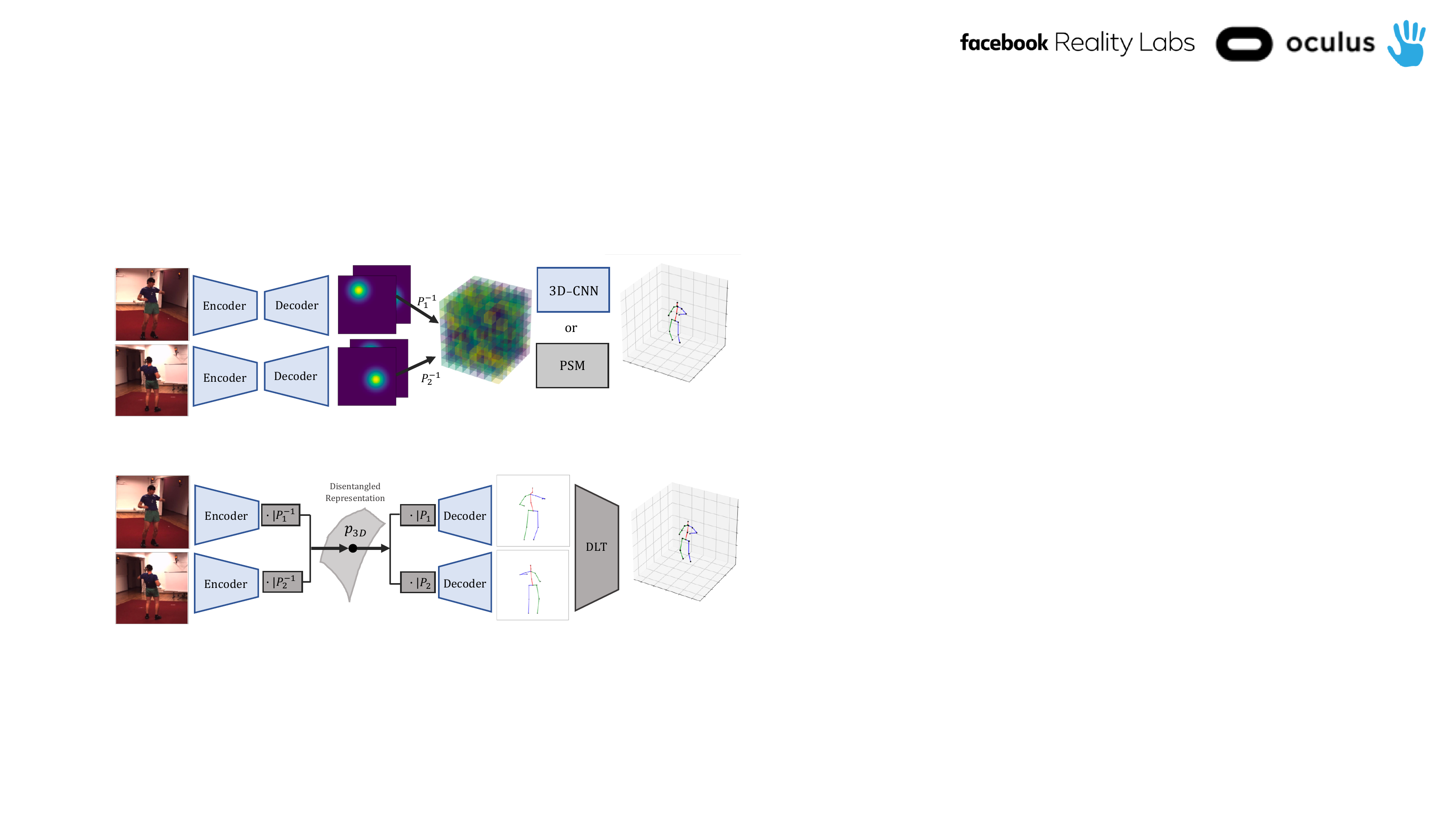}
		
		\put(24,29){a) \textbf{Volumetric Approaches}}
		\put(24,-2){b) \textbf{Canonical Fusion (ours)}}
		
	\end{overpic}

	\vspace{25pt}
	\caption{Overview of 3D pose estimation from multi-view images. The state-of-the-art approaches project 2D detections to 3D grids and reason jointly across views through computationally intensive volumetric convolutional neural networks \cite{Iskakov19} or Pictorial Structures (PSM) \cite{Qiu19, Pavlakos17}. This yields accurate predictions but is computationally expensive. We design a lightweight architecture that predicts 2D joint locations from a learned camera-independent representation of 3D pose and then lifts them to 3D via an efficient formulation of differentiable triangulation (DLT). Our method achieves performance comparable to volumetric methods, while, unlike them, working in real-time.
}
	\label{fig:teaser}
\end{figure}

Unfortunately, the state-of-the-art multi-camera 3D pose estimation algorithms tend to be computationally expensive because they rely on deep networks that operate on volumetric grids~\cite{Iskakov19}, or volumetric Pictorial Structures~\cite{Qiu19, Pavlakos17}, to combine features coming from different views in accordance with epipolar geometry. Fig.~\ref{fig:teaser}(a) illustrates these approaches. 

In this paper, we demonstrate that the expense of using a 3D grid is not required. Fig.~\ref{fig:teaser}(b) depicts our approach.
We encode each input image into latent representations, which are then efficiently transformed from image coordinates into world coordinates by conditioning on the appropriate camera transformation using feature transform layers~\cite{Worrall17}. This yields feature maps that live in a \textit{canonical} frame of reference and are \textit{disentangled} from the camera poses. The feature maps are fused using 1D convolutions into a unified latent representation, denoted as $p _{\text{3D}}$ in Fig.~\ref{fig:teaser}(b), which makes it possible to reason jointly about the extracted 2D poses across camera views. We then condition this latent code on the known camera transformation to decode it back to 2D image locations using a shallow 2D CNN. The proposed fusion technique, to which we will refer to as \textit{Canonical Fusion}, enables us to drastically improve the accuracy of the 2D detection compared to the results obtained from each image independently, so much so, that we can lift these 2D detections to 3D reliably using the simple Direct Linear Transform (DLT) method~\cite{Hartley03}. Because standard DLT implementations that rely on Singular Value Decomposition (SVD) are rarely efficient on GPUs, we designed a faster alternative implementation based on the Shifted Iterations method \cite{Quarteroni10}.

In short, our contributions are: (1) a novel multi-camera fusion technique that exploits 3D geometry in latent space to efficiently and jointly reason about different views and drastically improve the accuracy of 2D detectors, (2) a new GPU-friendly implementation of the DLT method, which is hundreds of times faster than standard implementations. 

We evaluate our approach on two large-scale multi-view datasets, Human3.6M~\cite{Ionescu14, IonescuSminchisescu11} and TotalCapture~\cite{Trumble17}: we outperform the state-of-the-art methods when additional training data is not available, both in terms of speed and accuracy. When additional 2D annotations can be used ~\cite{Lin14, Andriluka14}, our accuracy remains comparable to that of the state-of-the-art methods, while being faster.
Finally, we demonstrate that our approach can handle viewpoints that were never seen during training.
In short, we can achieve real-time performance without sacrificing prediction accuracy nor viewpoint flexibility, while other approaches cannot.

\section{Related Work}
Pose estimation is a long-standing problem in the computer vision community. In this section, we review in detail related multi-view pose estimation literature. We then focus on approaches lifting 2D detections to 3D via triangulation.
\\
\\
\textbf{Pose estimation from multi-view input images.}
Early attempts \cite{Liu11, Gall10, Burenius13, Belgiannis14} tackled pose-estimation from multi-view inputs by optimizing simple parametric models of the human body to match hand-crafted image features in each view, achieving limited success outside of the controlled settings.
With the advent of deep learning, the dominant paradigm has shifted towards estimating 2D poses from each view separately, through exploiting efficient monocular pose estimation architectures \cite{Newell16,Tompson15,Wei16, Sun19}, and then recovering the 3D pose from single view detections.

Most approaches use 3D volumes to aggregate 2D predictions. Pavlakos et al. \cite{Pavlakos17} project 2D keypoint heatmaps to 3D grids and use Pictorial Structures aggregation to estimate 3D poses.  
Similarly, \cite{Qiu19} proposes to use Recurrent Pictorial Structures to efficiently refine 3D pose estimations step by step. 
Improving upon these approaches, \cite{Iskakov19} projects 2D heatmaps to a 3D volume using a differentiable model and regresses the estimated root-centered 3D pose through a learnable 3D convolutional neural network. 
This allows them to train their system end-to-end by optimizing directly the 3D metric of interest through the predictions of the 2D pose estimator network. 
Despite recovering 3D poses reliably, volumetric approaches are computationally demanding, and simple triangulation of 2D detections is still the de-facto standard when seeking real-time performance \cite{Li19, Cao17}.

Few models have focused on developing lightweight solutions to reason about multi-view inputs.  In particular, \cite{Kadkhodamohammadi18} proposes to concatenate together pre-computed 2D detections and pass them as input to a fully connected network to predict global 3D joint coordinates. Similarly, \cite{Qiu19} refines 2D heatmap detections jointly by using a fully connected layer before aggregating them on 3D volumes.
Although, similar to our proposed approach, these methods fuse information from different views without using volumetric grids, they do not leverage camera information and thus overfit to a specific camera setting. We will show that our approach can handle different cameras flexibly and even generalize to unseen ones.
\\
\\
\textbf{Triangulating 2D detections.}
Computing the position of a point in 3D-space given its images in $n$ views and the camera matrices of those views is one of the most studied computer vision problems. We refer the reader to \cite{Hartley03} for an overview of existing methods. 
In our work, we use the Direct Linear Triangulation (DLT) method because it is simple and differentiable. We propose a novel GPU-friendly implementation of this method, which is up to two orders of magnitude faster than existing ones that are based on SVD factorization. We provide a more detailed overview about this algorithm in Section \ref{sec:DLT}. 

Several methods lift 2D detections efficiently to 3D by means of triangulation \cite{Amin13, Li19, Gunel19, Cao17}.
More closely related to our work, \cite{Iskakov19}  proposes to back-propagate through an SVD-based differentiable triangulation layer by lifting 2D detections to 3D keypoints. Unlike our approach, these methods do not perform any explicit reasoning about multi-view inputs and therefore struggle with large self-occlusions. 

\begin{figure*}[h]
	\centering
	\vspace{-20pt}
	\includegraphics[width=\textwidth]{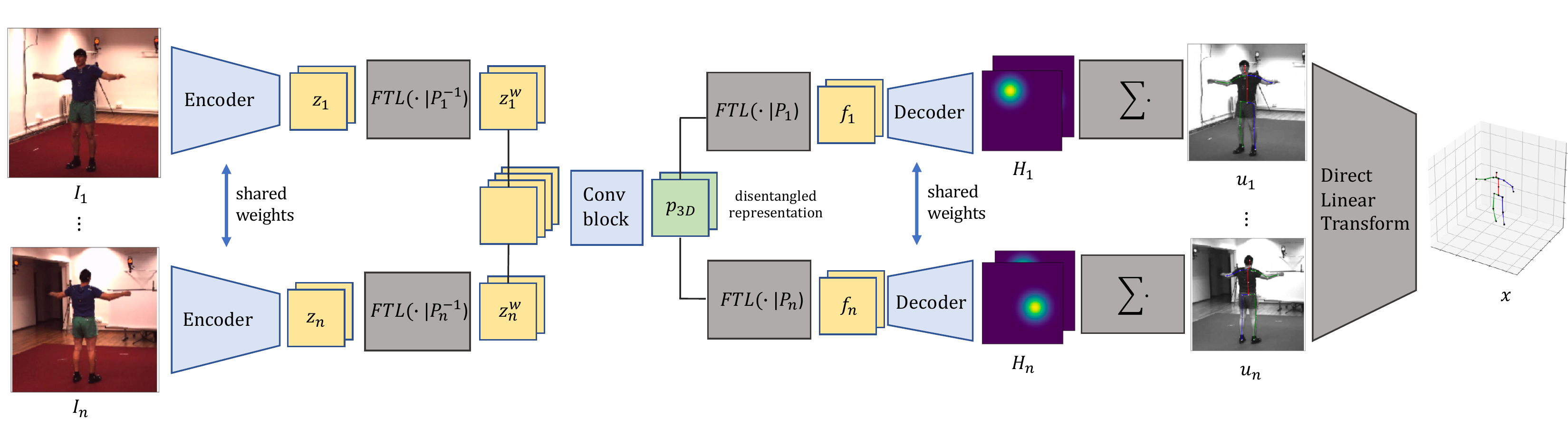}
	\vspace{-16pt}
	\caption{\textit{Canonical Fusion.} 
	The proposed architecture learns a unified view-independent representation of the 3D pose from multi-view inputs, allowing it to reason efficiently across multiple views.  Feature Transform Layers (FTL) use camera projection matrices ($P_i$) to map features between this canonical representation, while Direct Linear Transform (DLT) efficiently lifts 2D keypoints into 3D.  Blocks marked in gray are differentiable (supporting backpropagation) but not trainable.}
	\label{fig:arch}
\end{figure*}

\section{Method}
\label{sec:method}

We consider a setting in which $n$ spatially calibrated and temporally synchronized cameras capture the performance of a single individual in the scene.
We denote with $\{I_i\}_{i=1}^n$ the set of multi-view input images, each captured from a camera with known projection matrix $P_i$. 
Our goal is to estimate its 3D pose in the absolute world coordinates; we parameterize it as a fixed-size set of 3D point locations $\{\textbf x ^j\}_{j=1}^J$, which correspond to the joints.

Consider as an example the input images on the left of Figure \ref{fig:arch}. Although exhibiting different appearances, the frames share the same 3D pose information up to a perspective projection and view-dependent occlusions.
Building on this observation, we design our architecture (depicted in Figure \ref{fig:arch}), which learns a unified \textit{ view-independent representation} of 3D pose from multi-view input images.
This allows us to reason efficiently about occlusions to produce accurate 2D detections, that can be then simply lifted to 3D absolute coordinates by means of triangulation. 
Below, we first introduce baseline methods for pose estimation from multi-view inputs. We then describe our approach in detail and explain how we train our model.

\subsection{Lightweight pose estimation from multi-view inputs}

Given input images $\{I_i\}_{i=1}^n$, we use a convolutional neural network backbone to extract features $\{z_i\}_{i=1}^n$ from each input image separately. 
Denoting our encoder network as $e$, $z_i$ is computed as
\begin{align}
z_i = e (I_i).
\end{align}
Note that, at this stage, feature map $z_i$ contains a representation of the 3D pose of the performer that is fully \textit{entangled} with camera view-point, expressed by the camera projection operator $P_i$.

We first propose a baseline approach, similar to \cite{Li19, Gunel19}, to estimate the 3D pose from multi-view inputs. Here, we simply decode latent codes $z_i$ to 2D detections, and lift 2D detections to 3D by means of triangulation. We refer to this approach as \textit{Baseline}. Although efficient, we argue that this approach is limited because it processes each view independently and therefore cannot handle self-occlusions.

An intuitive way to jointly reason across different views is to use a learnable neural network to share information across embeddings $\{z_i\}_{i=1}^n$, by concatenating features from different views and processing them through convolutional layers into view-dependent features, similar in spirit to the recent models \cite{Kadkhodamohammadi18, Qiu19}. In Section \ref{sec:exp} we refer to this general approach as \textit{Fusion}. 
Although computationally lightweight and effective, we argue that this approach is limited for two reasons: (1) it does not make use of known camera information, relying on the network to learn the spatial configuration of the multi-view setting from the data itself, and (2) it cannot generalize to different camera settings by design. 
We will provide evidence for this in Section \ref{sec:exp} .

\subsection{Learning a view-independent representation}

To alleviate the aforementioned limitations, we propose a method to jointly reason across views, leveraging the observation that the 3D pose information contained in feature maps $\{z_i\}_{i=1}^n$ is the same across all $n$ views up to camera projective transforms and occlusions, as discussed above. 
We will refer to this approach as \textit{Canonical Fusion}.

To achieve this goal, we leverage \textit{feature transform layers} (FTL)  \cite{Worrall17}, which was originally proposed as a technique to condition latent embeddings on a target transformation so that to learn interpretable representations. 
Internally, a FTL has no learnable parameter and is computationally efficient. It simply reshapes the input feature map to a point-set, applies the target transformation, and then reshapes the point-set back to its original dimension.
This technique forces the learned latent feature space to preserve the structure of the transformation, resulting in practice in a disentanglement between the learned representation and the transformation.
In order to make this paper more self-contained, we review FTL in detail in the Supplementary Section. 

Several approaches have used FTL for novel view synthesis to map the latent representation of images or poses from one view to another \cite{Rhodin18, Rhodin19, Chen19b, Chen19a}. In this work, we leverage FTL to map images from multiple views to a unified latent representation of 3D pose.
In particular, we use FTL to project feature maps $z_i$ to a common canonical representation by explicitly conditioning them on the camera projection matrix $P_i ^{-1}$ that maps image coordinates to the world coordinates
\begin{align}
z_i ^ w= \text{FTL} (z_i |P_i ^{-1}).
\end{align}

Now that feature maps have been mapped to the same canonical representation, they can simply be concatenated and fused into a \textit{unified representation of 3D pose} via a shallow 1D convolutional neural network $f$, i.e. 
\begin{align}
p_{\text{3D}} = f (\text{concatenate}(\{z_i ^w \}_{i=1}^n ) ).
\end{align}

We now force the learned representation to be disentangled from camera view-point by transforming the shared $p_{\text{3D}}$ features to view-specific representations $f_i$ by
\begin{align}
f_i = \text{FTL} (p_{\text{3D}}| P_i ).
\end{align}
In Section \ref{sec:exp}  we show both qualitatively and quantitatively that the representation of 3D pose we learn is effectively disentangled from the camera-view point.
\begin{algorithm}[t]
	\label{alg:algo}
	\SetAlgoLined
	$A \gets A(\{\textbf u_i, P_i \} _{i=1} ^ N)$\;
	$B \gets (A^TA + \sigma I) ^{-1}$\;
	$\sigma \gets 0.001$ (see  Theorem \ref{thm:sing})\;
	$ \textbf  x \gets \text{rand}(4,1) $\;
	\For{$i = 1:T$}
	{
		$\textbf x \gets B \textbf x $\;
		$\textbf x \gets \textbf x / \|\textbf x\|$\;
	}
      \textbf{return} $\textbf y \gets \textbf x(0:3)/ \textbf x(4)$\;
	\caption{DLT-SII$(\{\textbf u_i, P_i \} _{i=1} ^ N, T=2)$}
\end{algorithm}

Unlike the \textit{Fusion} baseline, \textit{Canonical Fusion} makes explicit use of camera projection operators to simplify the task of jointly reasoning about views. The convolutional block, in fact, now does not have to figure out the geometrical disposition of the multi-camera setting and can solely focus on reasoning about occlusion. 
Moreover, as we will show, \textit{Canonical Fusion} can handle different cameras flexibly, and even generalize to unseen ones.

\subsection{Decoding latent codes to 2D detections}
This component of our architecture proceeds as a monocular pose estimation model that maps view-specific representations $f_i$ to 2D Heatmaps $H_i$ via a shallow convolutional decoder $d$, i.e.
\begin{align}
H_i ^ j = d(f_i ),
\end{align}
where $H_i ^ j$ is the heatmap prediction for joint $j$ in Image $i$.
Finally, we compute the 2D location $u _i ^ j$ of each joint $j$ by simply integrating heatmaps across spatial axes
\begin{align}
\textbf u _i ^ j= \Big(\sum_{x,y}  x H_i ^j , \sum_{x,y}  y H_i ^j  \Big)  / \sum_{x,y}  H_i^j.        
\end{align}
Note that this operation is differentiable with respect to heatmap $H_i ^j$, allowing us to back-propagate through it.
In the next section, we explain in detail how we proceed to lift multi-view 2D detections to 3D.

\begin{figure}[t]
	\centering
	\begin{overpic}[clip,  trim=0.5cm 0.5cm 0.0cm 0cm,width=0.5\textwidth]{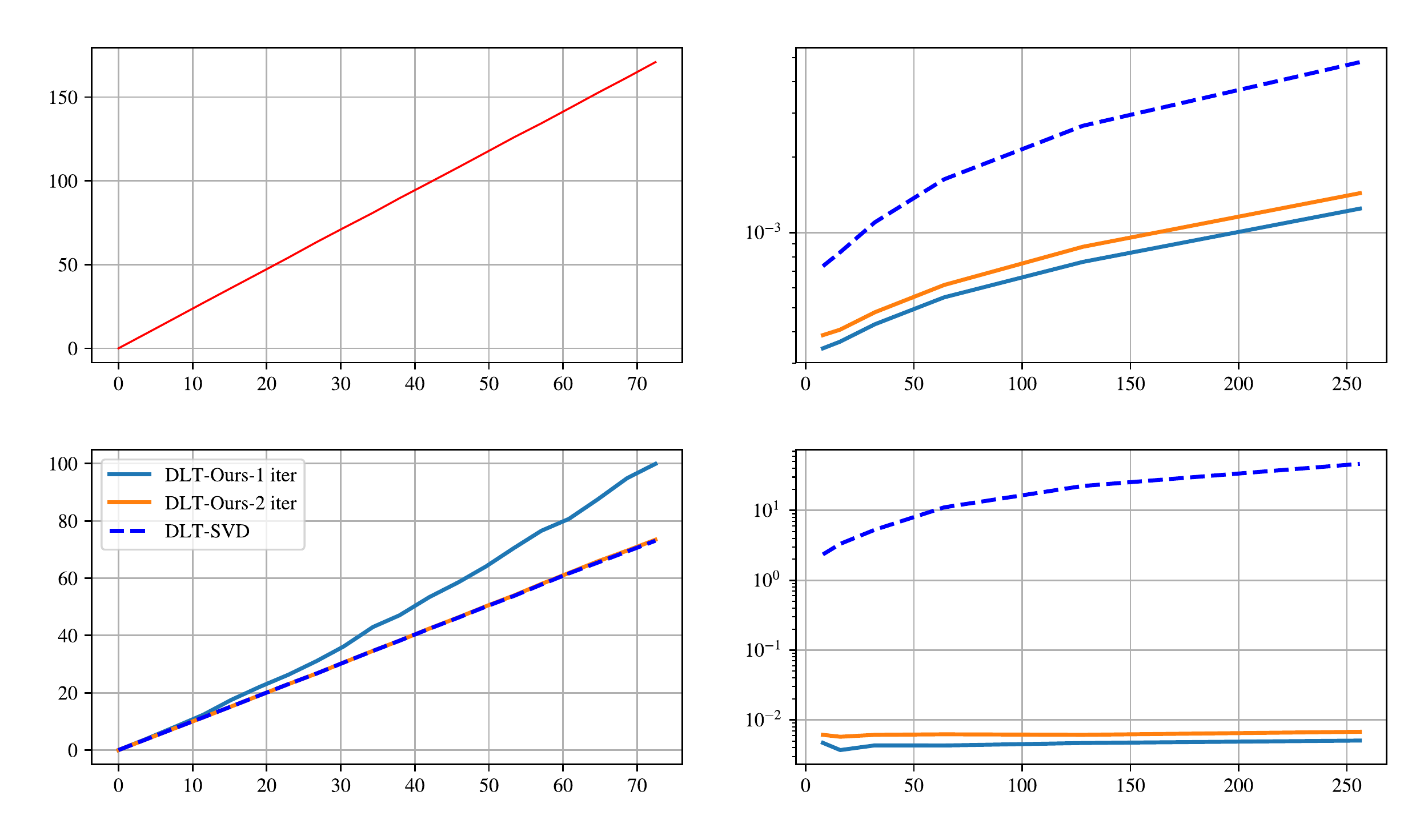}
		
		\put(17,56){ \footnotesize{a) Theorem \ref{thm:sing} }}
		\put(65.5,56) {\footnotesize{c) CPU Profiling }}
	\put(19,26.7) {\footnotesize{b) Accuracy }}
	\put(64,26.7){ \footnotesize{ d) GPU Profiling  }}
	
		\put(19,-2.0) {\footnotesize{2D-MPJPE }}
		\put(71,-2.0) {\footnotesize{batch size}}
		
		\put (-2,35.5) {\footnotesize{\rotatebox{90}{$\mathbb E [\sigma_{\text{min}}(A^*)]$}}}
		\put (-2,6.5) {\footnotesize{\rotatebox{90}{\footnotesize{3D-MPJPE }}}}
			\put (49,10.5) {\footnotesize{\rotatebox{90}{\footnotesize{time(s)  }}}}
				\put (49,38.5) {\footnotesize{\rotatebox{90}{\footnotesize{time(s) }}}}

	\end{overpic}

	\vspace{8pt}
	\caption{Evaluation of DLT. We validate the findings of Theorem \ref{thm:sing} in (a). We then compare our proposed DLT implementation to the SVD of \cite{Iskakov19}, both in terms of accuracy (b) and performance (c),(d). Exploiting Theorem 1, we can choose a suitable approximation for $\sigma_{\text{min}}(A^*)$, and make DLT-SII converge to the desired solution in only two iterations. }
	\label{fig:dlt}
\end{figure}
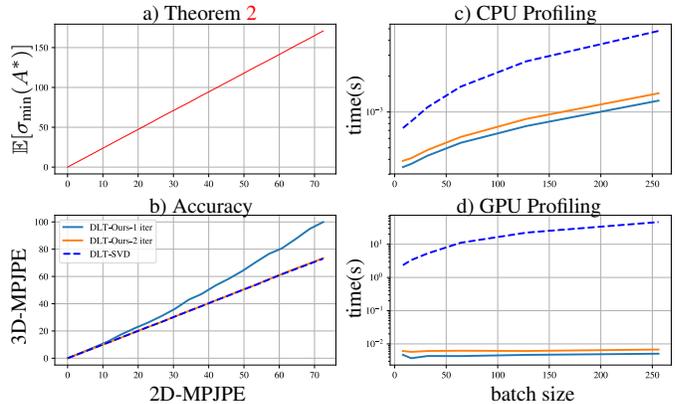

\subsection{Efficient Direct Linear Transformation}
\label{sec:DLT}

In this section we focus on finding the position $\textbf x^j=[x^j,y^j,z^j]^T$ of a 3D point in space given a set of $n$ 2d detections $\{\textbf u_i^j\}_{i=1}^n$. To ease the notation, we will drop apex $j$ as the derivations that follow are carried independently for each landmark.

Assuming a pinhole camera model, we can write $d_i \textbf u_i = P_i \textbf x$, where $d_i$ is an unknown scale factor. 
Note that here, with a slight abuse of notation, we express both 2d detections $\textbf u_i$ and 3d landmarks $\textbf x$  in homogeneous coordinates. 
Expanding on the components we get
\begin{align}
\label{eq:pinhole}
d_i u_i = p_i ^ {1T} \textbf x \, , \, d_i v_i = p_i ^ {2T} \textbf x \, , \, d_i = p_i ^ {3T} \textbf x ,
\end{align}
where $p_i ^ {kT}$ denotes the $k$-th row of $i$-th camera projection matrix.
Eliminating $d_i$ using the third relation in (\ref{eq:pinhole}), we obtain 
\begin{align}
(u_i p_i ^ {3T} - p_i ^ {1T}) &\textbf x = 0 \label{eq:dlt_x}\\
(v_i p_i ^ {3T} - p_i ^ {2T}) &\textbf x = 0. \label{eq:dlt_y}
\end{align}
Finally, accumulating over all available $n$ views yields a total of $2n$ linear equations in the unknown 3D position $ \textbf x $, which we write compactly as 
\begin{align}
A\textbf x = \textbf 0,
 \, \,  \, \, \text{where} \, \,A= A(\{ u_i, v_i, P_i\} _{i=1}^N) .
\label{eq:dlt}
\end{align}
Note that $A\in \mathbb R ^{2n \times 4}$ is a function of $\{ u_i, v_i, P_i\} _{i=1}^N$, as specified in Equations (\ref{eq:dlt_x}) and (\ref{eq:dlt_y}). We refer to $A$ as the DLT matrix.
These equations define $\textbf x$ up to a scale factor, and we seek a non-zero solution.
In the absence of noise, Equation (\ref{eq:dlt}) admits a unique non-trivial solution, corresponding to the 3D intersection of the camera rays passing by each 2D observation $\textbf u_i$ (i.e. matrix $A$ does not have full rank).
However, considering noisy 2D point observations such as the ones predicted by a neural network, Equation (\ref{eq:dlt}) does not admit solutions, thus we have to seek for an approximate one. 
A common choice, known as the \textit{Direct Linear Transform} (DLT) method \cite{Hartley03}, proposes the following relaxed version of Equation (\ref{eq:dlt}):
\begin{align}
\text{min}_{\textbf x} \| A\textbf x \| , \text{subject to} \|\textbf x \| = 1.
\label{eq:dlt_rel}
\end{align}
Clearly, the solution to the above optimization problem is the eigenvector of $A^TA$ associated to its smallest eigenvalue $\lambda _{\text{min}}(A^TA)$. 
In practice, the eigenvector is computed by means of Singular Value Decomposition (SVD) \cite{Hartley03}. We argue that this approach is suboptimal, as we in fact only care about \textit{one} of the eigenvectors of $A^TA$.

Inspired by the observation above that the smallest eigenvalue of $A^TA$ is zero for non-noisy observations, we derive a bound for the smallest eigenvalue of matrix $A^TA$ in the presence of Gaussian noise.  We prove this estimate in the Supplementary Section.\\
\begin{theorem}
	\label{thm:sing}
	Let $A$ be the DLT matrix associated to the non-perturbed case, i.e. $\sigma_{\text{min}}(A)  = 0$. 
	Let us assume i.i.d Gaussian noise $\varepsilon =(\varepsilon_u,\varepsilon_v)\sim \mathcal N(0,s^2I)$ in our 2d observations, i.e. $(u^*, v^*) = (u + \varepsilon _u, v + \varepsilon _v)$, and let us denote as $A^*$ the DLT matrix associated  to the perturbed system. 
	Then, it follows that:
	\begin{align}
	0 \leq \mathbb E [\sigma_{\text{min}}(A^*)] \leq  C s, \, \, \text{where } C=C(\{u_i, P_i \}_{i=1}^N)
	\end{align}
\end{theorem}

In Figure \ref{fig:dlt}(a) we reproduce these setting by considering Gaussian perturbations of 2D observations, and find an experimental confirmation that by having a greater 2D joint measurement error, specified by 2D-MPJPE (see Equation \ref{eq:loss2d} for its formal definition), the expected smallest singular value $\sigma_{\text{min}}(A^*)$ increases linearly.

The bound above, in practice, allows us to compute the smallest singular vector of $A^{*}$ reliably by means of \textit{Shifted Inverse Iterations} (SII) \cite{Quarteroni10}: we can estimate $\sigma_{\text{min}}(A^*)$ with a small constant and know that the iterations will converge to the correct eigenvector. 
For more insight on why this is the case, we refer the reader to the Supplementary Section.

SII can be implemented extremely efficiently on GPUs. As outlined in Algorithm \ref{alg:algo}, it consists of one  inversion of a $4\times 4$ matrix and several matrix multiplication and vector normalizations, operations that can be trivially parallelized. 
In Figure \ref{fig:dlt}(b) we compare our SII based implementation of DLT (estimating the smallest singular value of $A$ with $\sigma = 0.001$) to an SVD  based one, such as the one proposed in \cite{Iskakov19}. For 2D observation errors up to $70$ pixels (which is a reasonable range in $256$ pixel images), our formulation requires as little as two iterations to achieve the same accuracy as a full SVD factorization, while being respectively $10/100$ times faster on CPU/GPU than its counterpart, as evidenced by our profiling in Figures \ref{fig:dlt}(c,d).

\subsection{Loss function}

In this section, we explain how to train our model. Since our DLT implementation is differentiable with respect to 2D joint locations $\textbf u_i$, we can let gradients with respect to 3D landmarks $\textbf x$ flow all the way back to the input images $\{I_i\}_{i=1}^n$, making our approach trainable end-to-end.
However, in practice, to make training more stable in its early stages, we found it helpful to first train our model by minimizing a 2D Mean Per Joint Position Error (MPJPE) of the form
\vspace{-6pt}
\begin{align}
\label{eq:loss2d}
L_{ \text{2D-MPJPE} } =  \sum_{i=1}^n  \frac{1}{J} \sum_{j=1} ^ J \| \mathbf u_i ^j -  \mathbf {\hat u} _i ^j \|_2,
\end{align}
where $\mathbf {\hat u}_j ^ i $ denotes the ground truth 2D position of $j$-th joint in the $i$-th image.
In our experiments, we pre-train our models by minimizing $L_{ \text{2D-MPJPE} } $ for 20 epochs.
Then, we fine-tune our model by minimizing 3D MPJPE, which is also our test metric, by
\vspace{-6pt}
\begin{align}
L_{ \text{3D-MPJPE} } =  \frac{1}{J} \sum_{j=1} ^ J \| \mathbf x _j -  \mathbf {\hat x}_j \|_2,
\end{align}
where $\mathbf {\hat x}_j $ denotes the ground truth 3D position of $j$-th joint in the world coordinate.
We evaluate the benefits of fine-tuning using $L_{ \text{3D-MPJPE} } $ in the Section \ref{sec:exp}.

\section{Experiments}
\label{sec:exp} 
We conduct our evaluation on two available large-scale multi-view datasets, TotalCapture \cite{Trumble17}  and Human3.6M \cite{Ionescu14}.
We crop each input image around the performer, using ground truth bounding boxes provided by each dataset.
Input crops are undistorted, re-sampled so that virtual cameras are pointing at the center of the crop and normalized to $256\times256$.
We augment our train set by performing random
rotation($\pm 30$ degrees, note that image rotations correspond to camera rotations along the z-axis) and standard color augmentation.
In our experiments, we use a ResNet152 \cite{He16} pre-trained on ImageNet \cite{Deng09} as the backbone architecture for our encoder. Our fusion block consists of two $1\times1$ convolutional layers. Our decoder consists of 4 transposed convolutional layers, followed by a $1\times1$ convolution to produce heatmaps. More details on our architecture are provided in the Supplementary section.
The networks are trained for $50$ epochs, using a Stochastic Gradient Descent optimizer where we set learning rate to $2.5\times 10^{-2}$.

\begin{figure}[t]
	\centering
	\hspace{-14pt}
	\begin{overpic}[clip, trim=0.0cm 0cm 0cm 0cm,width=0.5\textwidth]{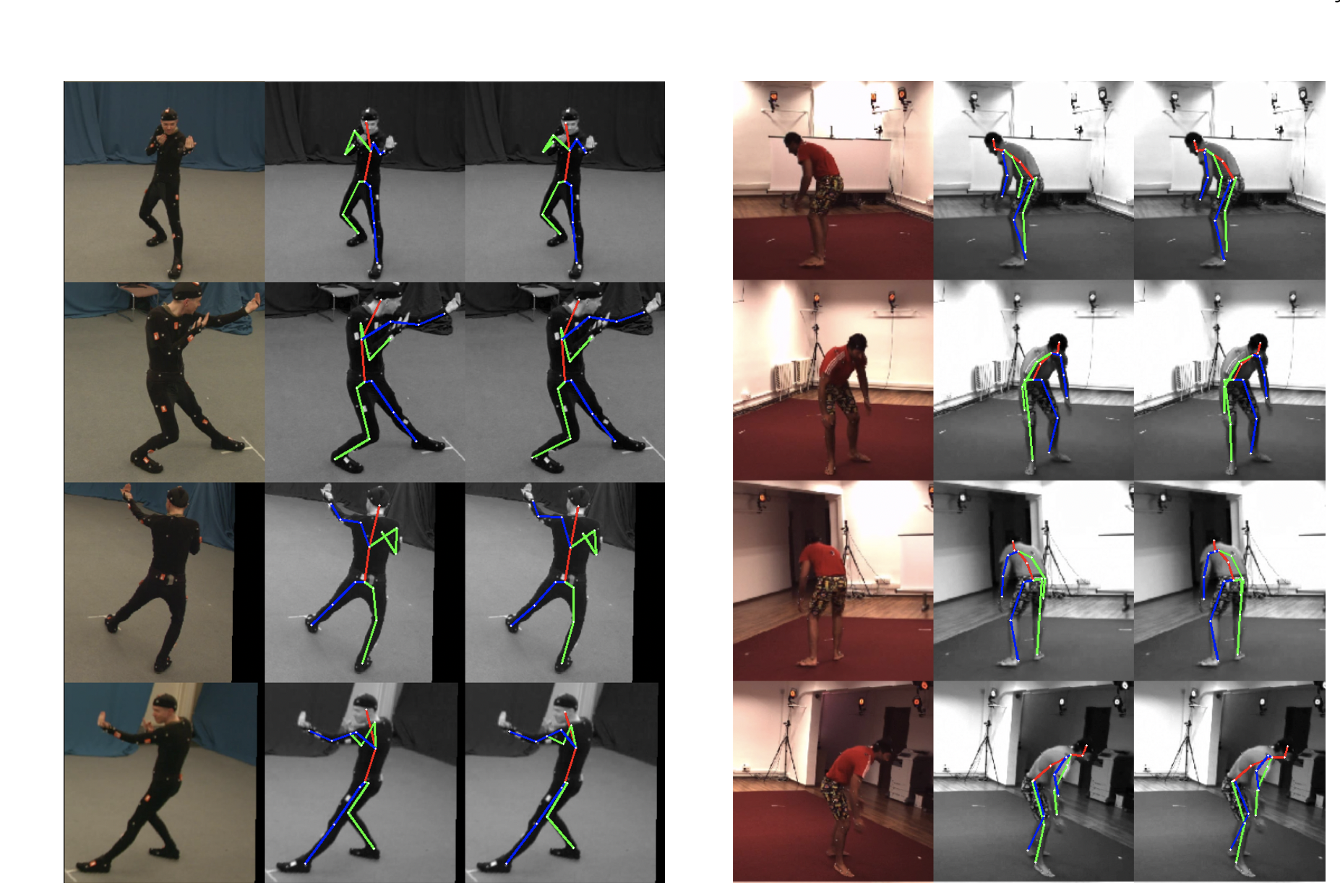}
		\put (15.5,68) {\small{a) \textbf{Total Capture}}}
		\put (12,62) {\small{$I_i$}}
		\put (26,62) {\small{$\mathbf {\hat u} _i$}}
		\put (40,62) {\small{$\mathbf {u} _i$}}
		
		\put (63.5,68) {\small{b) \textbf{Human3.6M}}}
		\put (62,62) {\small{$I_i$}}
		\put (74,62) {\small{$\mathbf {\hat u} _i$}}
		\put (89,62) {\small{$\mathbf {u} _i$}}
	
		\put (1,1.5) {\footnotesize{\rotatebox{90}{camera 7}}}
		\put (1,16.5) {\footnotesize{\rotatebox{90}{camera 5}}}
		\put (1,31.5) {\footnotesize{\rotatebox{90}{camera 3}}}
		\put (1,47.5) {\footnotesize{\rotatebox{90}{camera 1}}}

		\put (51,1.5) {\footnotesize{\rotatebox{90}{camera 4}}}
		\put (51,16.5) {\footnotesize{\rotatebox{90}{camera 3}}}
		\put (51,31.5) {\footnotesize{\rotatebox{90}{camera 2}}}
		\put (51,47.5) {\footnotesize{\rotatebox{90}{camera 1}}}
	
	\end{overpic}
	\caption{We visualize randomly picked samples from the test set of TotalCapture and Human3.6M.
		To stress that the pose representation learned by our network is effectively \textit{disentangled} from the camera view-point, we intentionally show predictions \textit{before} triangulating them, rather than re-projecting triangulated keypoints
to the image space. Predictions are best seen in supplementary videos.}
	\label{fig:vis}
\end{figure}
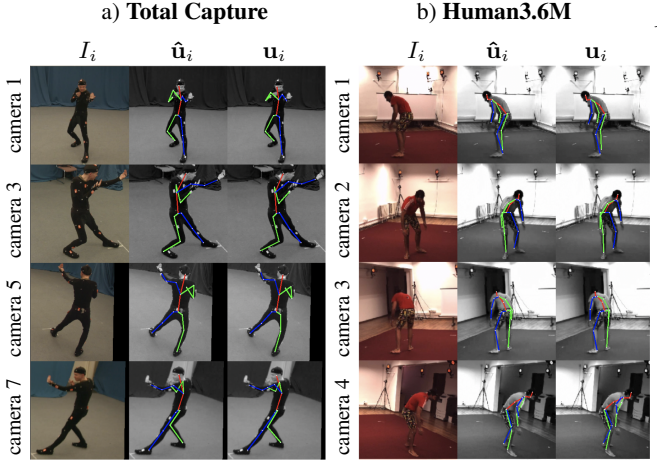

\subsection{Datasets specifications}
\textbf{TotalCapture}: The TotalCapture dataset \cite{Trumble17} has been recently introduced to the community.
It consists of 1.9 million frames, captured from 8 calibrated full HD video cameras recording at 60Hz. 
It features 4 male and 1 female subjects, each performing five diverse performances repeated 3 times: ROM, Walking, Acting, Running, and Freestyle.
Accurate 3D human joint locations are obtained from a marker-based motion capture system. 
Following previous work \cite{Trumble17}, the training set consists of “ROM1,2,3”,
“Walking1,3”, “Freestyle1,2”, “Acting1,2”, “Running1” on
subjects 1,2 and 3. The testing set consists of “Walking2 (W2)”, “Freestyle3 (FS3)”, and “Acting3 (A3)” on subjects 1, 2, 3, 4, and 5. The number following each action indicates the video of that action being used, for example Freestyle has three videos of the same action of which 1 and 2 are used for training and 3 for testing.  This setup allows for testing on unseen and seen subjects but always unseen performances.
Following \cite{Qiu19}, we use the data of four cameras (1,3,5,7) to train and test our models.
However, to illustrate the generalization ability of our approach to new camera settings, we propose an experiment were we train on cameras (1,3,5,7) and test on \textit{unseen} cameras (2,4,6,8).
\\
\\
\textbf{Human 3.6M}:  The Human3.6M dataset \cite{Ionescu14} is the largest publicly available 3D human pose estimation benchmark. 
It consists of 3.6 million frames, captured from 4 synchronized 50Hz digital cameras.  
Accurate 3D human joint locations are obtained from a marker-based motion capture system utilizing 10 additional IR sensors. 
It contains a total of 11 subjects (5 females and 6 males) performing 15 different activities. 
For evaluation, we follow the most popular protocol, by training on subjects 1, 5, 6, 7, 8 and using unseen subjects 9, 11 for testing. Similar to other methods \cite{Martinez17, Pavlakos17, Tome18, Kadkhodamohammadi18, Qiu19}, we use all available views during training and inference.

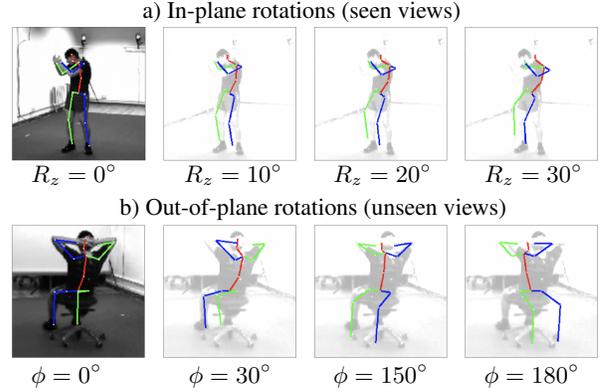
\begin{figure}[t]
	\vspace{-0.5cm}
	\centering
	\begin{overpic}[clip, trim=0.0cm 0cm 0cm 1.5cm,width=0.5\textwidth]{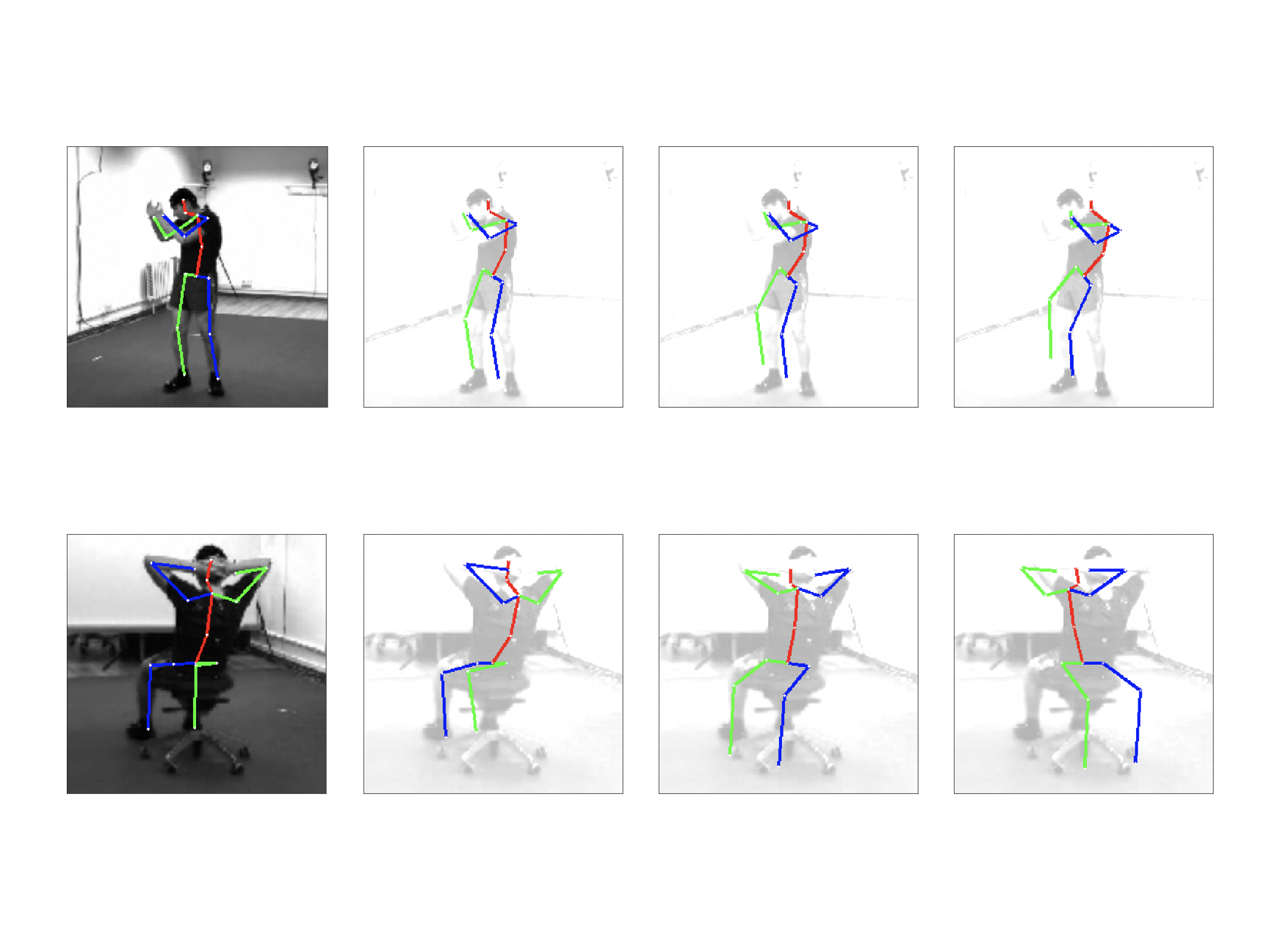}
		\put (25,62) {\small{a) In-plane rotations (seen views)}}
		\put (8,37) {\small{$R_z = 0^{\circ}$}}
		\put (31.3,37) {\small{$R_z = 10^{\circ}$}}
		\put (54,37) {\small{$R_z = 20^{\circ}$}}
		\put (77,37) {\small{$R_z = 30^{\circ}$}}
		
		\put (21.5,32) {\small{b) Out-of-plane rotations (unseen views)}}
		\put (8,6.5) {\small{${\phi} = 0^{\circ}$}}
		\put (32,6.5) {\small{${\phi} = 30^{\circ}$}}
		\put (54,6.5) {\small{${\phi} = 150^{\circ}$}}
		\put (77,6.5) {\small{${\phi} = 180^{\circ}$}}
	\end{overpic}
\vspace{-20pt}
	\caption{In the top row, we synthesize 2D poses after rotating cameras with respect to z-axis. In the bottom row, we rotate camera around the plane going through two consecutive camera views by angle $\phi$, presenting the network with \textit{unseen} camera projection matrices. Note that after decoding $p_{\text{3D}}$ to a novel view, it no longer corresponds to the encoded view. 2D Skeletons are overlaid on one of the original view in order to provide a reference. These images show that the 3D pose embedding $p_{\text{3D}}$ is \textit{disentangled} from the camera view-point. Best seen in supplementary videos.}
	\label{fig:dise}
\end{figure}

\begin{table*}[t]
	\begin{center}
			\scalebox{0.9}{
		\begin{tabular}{c|ccc|ccc|c}
			\hline
			Methods & \multicolumn{3}{c|}{Seen Subjects (S1,S2,S3)} & \multicolumn{3}{c|}{Unseen Subjects (S4,S5)} & Mean\\
			& Walking & Freestyle & Acting & Walking & Freestyle & Acting & \\
			\hline
			Qui \etal \cite{Qiu19} Baseline + RPSM &28 &42 &30 &45 &74 &46 &41 \\
			Qui \etal \cite{Qiu19} Fusion + RPSM &19 & \textbf{28} & 21 & 32 & \textbf{54} & \textbf{33} & 29 \\
			\hline
			Ours, Baseline & 31.8 & 36.4 & 24.0 & 43.0 & 75.7 & 43.0 & 39.3  \\
			Ours, Fusion & 14.6 & 35.3 & 20.7 & 28.8 & 71.8 & 37.3 & 31.8  \\
			Ours, Canonical Fusion(no DLT) & 10.9 & 32.2 & 16.7 & 27.6 & 67.9 & 35.1 & 28.6  \\
			Ours, Canonical Fusion& \textbf{10.6} & 30.4 & \textbf{16.3} & \textbf{27.0} & 65.0 & 34.2 & \textbf{27.5} \\
			\hline
		\end{tabular}}
	\end{center}
	\caption{3D pose estimation error MPJPE (mm) on the TotalCapture dataset. The results reported for our methods are obtained without rigid alignment or further offline post-processing.}
	\label{tab:total_cap}
\end{table*}

\begin{table*}[h]
	\begin{center}
		\begin{tabular}{c|ccc|ccc|c }
			\hline
			Methods & \multicolumn{3}{c|}{Seen Subjects (S1,S2,S3)} & \multicolumn{3}{c|}{Unseen Subjects (S4,S5)} & Mean  \\
			& Walking & Freestyle & Acting & Walking & Freestyle & Acting & \\
			\hline
			Ours, Baseline & 28.9 & 53.7 & 42.4 & 46.7 & 75.9 & 51.3 & 48.2  \\
			Ours, Fusion & 73.9 & 71.5 & 71.5 & 72.0 & 108.4 & 58.4 & 78.9 \\
			Ours, Canonical Fusion & \textbf{22.4} & \textbf{47.1} & \textbf{27.8} & \textbf{39.1} & \textbf{75.7} & \textbf{43.1} & \textbf{38.2} \\
			\hline
		\end{tabular}
	\end{center}
	\caption{Testing the generalization capabilities of our approach on unseen views.
		We take the networks of Section \ref{sec:ablate}, trained on cameras (1,3,5,7) of the TotalCapture training set, and test on the unseen views captured with cameras (2,4,6,8). We report 3D pose estimation error MPJPE (mm).}
	\label{tab:total_cap_unseen}
\end{table*}

\subsection{Qualitative evaluation of disentanglement}

We evaluate the quality of our latent representation by showing that 3D pose information is effectively disentangled from the camera view-point. 
Recall from Section \ref{sec:method} that our encoder $e$ encodes input images to latent codes $z_i$, which are transformed from camera coordinates to the world coordinates and latter fused into a unified representation $p_{\text{3D}}$ which is meant to be disentangled from the camera view-point.
To verify this is indeed the case, we propose to decode our representation to different 2D poses by using different camera transformations $P$, in order to produce views of the same pose from novel camera view-points.
We refer the reader to Figure \ref{fig:dise} for a visualization of the synthesized poses.  In the top row, we rotate one of the cameras with respect to the z-axis, presenting the network with projection operators that have been seen at train time. 
In the bottom row we consider a more challenging scenario, where we synthesize novel views by rotating the camera around the plane going through two consecutive camera views.
Despite presenting the network with unseen projection operators, our decoder is still able to synthesize correct 2D poses.
This experiment shows our approach has effectively learned a representation of the 3D pose that is disentangled from camera view-point. We evaluate it quantitatively in Section \ref{sec:eval_unseen}.
\begin{table*}[h]
	\begin{center}
		\scalebox{0.7}{
			\begin{tabular}{c|ccccccccccccccc|c}
				\hline
				Methods & Dir. & Disc. & Eat & Greet & Phone & Photo & Pose & Purch. &  Sit & SitD. & Smoke & Wait & WalkD. & Walk & WalkT. & Mean \\
				\hline
				Martinez \etal \cite{Martinez17} & 46.5 & 48.6 &54.0 &51.5 &67.5 &70.7 &48.5 &49.1 &69.8 &79.4 &57.8 &53.1 &56.7 &42.2 &45.4 &57.0   \\
				Pavlakos \etal \cite{Pavlakos17} &41.2 &49.2 &42.8 &43.4 &55.6 &46.9 &40.3 &63.7 &97.6 &119.0 &52.1 &42.7 &51.9 &41.8 &39.4 & 56.9  \\
				Tome \etal \cite{Tome18} &43.3& 49.6& 42.0 &48.8 &51.1 &64.3 &40.3 &43.3 &66.0 &95.2 &50.2 &52.2 &51.1 &43.9 &45.3 &  52.8 \\
				Kadkhodamohammadi \etal \cite{Kadkhodamohammadi18} & 39.4 &46.9 &41.0 &42.7 &53.6 &54.8& 41.4 &50.0 &59.9 &78.8 &49.8 &46.2 & 51.1 &40.5 &41.0 &49.1 \\
				Qiu \etal \cite{Qiu19}  & 34.8 &35.8&32.7 &33.5 &34.5 &38.2 &29.7 &60.7 & 53.1 &35.2 &41.0 &41.6 &31.9 &31.4 & 34.6 &38.3  \\
				Qui \etal \cite{Qiu19}  + RPSM &28.9 &32.5 &26.6 &28.1 &\textbf{28.3} &\textbf{29.3} &\textbf{28.0} &36.8  &41.0 &\textbf{30.5} &35.6 &30.0 &28.3 &\textbf{30.0} &30.5 &31.2 \\
				\hline
				Ours, Baseline & 39.1 & 46.5 & 31.6 & 40.9 & 39.3 & 45.5 & 47.3 & 44.6  & 45.6 & 37.1 & 42.4 & 46.7 & 34.5 & 45.2 & 64.8 & 43.2  \\ 
				Ours, Fusion & 31.3 & 37.3 & 29.4 & 29.5 & 34.6 & 46.5 & 30.2 & 43.5  & 44.2 & 32.4 & 35.7 & 33.4 & 31.0 & 38.3 & 32.4 & 35.4  \\
				Ours, Canonical Fusion (no DLT) & 31.0 & 35.1 & 28.6 & 29.2 & 32.2 & 34.8 & 33.4 & 32.1  & 35.8 & 34.8 & 33.3 & 32.2 & 29.9 & 35.1 & 34.8 & 32.5  \\
				Ours, Canonical Fusion & \textbf{27.3} & \textbf{32.1} & \textbf{25.0} & \textbf{26.5} & 29.3 & 35.4 & 28.8 & \textbf{31.6}  & \textbf{36.4} & 31.7 & \textbf{31.2} & \textbf{29.9} &  \textbf{26.9}& 33.7 & \textbf{30.4} &  \textbf{30.2}  \\
				\hline
		\end{tabular}}
	\end{center}
	\caption{No additional training data setup. We compare the 3D pose estimation error  (reported in MPJPE (mm)) of our method to the state-of-the-art approaches on the Human3.6M dataset. The reported results for our methods are obtained without rigid alignment or further offline post-processing steps.}
	\label{tab:h36}
\end{table*}
\vspace{-2pt}
\subsection{Quantitative evaluation on TotalCapture}
\label{sec:ablate}

We begin by evaluating the different components of our approach and comparing to the state-of-the-art volumetric method of \cite{Qiu19} on the TotalCapture dataset. We report our results in Table \ref{tab:total_cap}.
We observe that by using the feature fusion technique (\textit{Fusion}) we get a significant $19 \%$ improvement
over our \textit{Baseline}, showing that, although simple, this fusion technique is effective.
Our more sophisticated \textit{Canonical Fusion (no DLT)} achieves further $10 \%$ improvement, showcasing that our method can effectively use camera projection operators to better reason about views. 
Finally, training our architecture by back-propagating through the triangulation layer (\textit{Canonical Fusion}) allows to further improve our accuracy by  $3 \%$. This is not surprising as we optimize directly for the target metric when training our network.
Our best performing model outperforms the state-of-the-art volumetric model of \cite{Qiu19} by $\sim5 \%$. Note that their method lifts 2D detections to 3D using Recurrent Pictorial Structures (RPSM), which uses a pre-defined skeleton, as a strong prior, to lift 2D heatmaps to 3D detections. Our method doesn't use any priors, and still outperform theirs. Moreover, our approach is orders of magnitude faster than theirs, as we will show in Section \ref{sec:eval_add}. We show some uncurated test samples from our model in Figure \ref{fig:vis}(a).

\subsection{Generalization to unseen cameras}
\label{sec:eval_unseen}

To assess the flexibility of our approach, we evaluate its performance on images captured from unseen views. To do so, we take the trained network of Section \ref{sec:ablate} and test it on cameras (2,4,6,8). Note that this setting is particularly challenging not only because of the novel camera views, but also because the performer is often out of field of view in camera $2$. 
For this reason, we discard frames where the performer is out of field of view when evaluating our \textit{Baseline}.
We report the results in Table \ref{tab:total_cap_unseen}. 
We observe that \textit{Fusion} fails at generalizing to novel views (accuracy drops by $47.1$mm when the network is presented with new views). This is not surprising as this fusion technique over-fits by design to the camera setting.
On the other hand the accuracy drop of  \textit{Canonical Fusion} is similar to the one of \textit{Baseline} ($\sim10$mm).
Note that our comparison favors \textit{Baseline} by discarding frames when object is occluded.
This experiments validates that our model is able to cope effectively with challenging unseen views.

\subsection{Quantitative evaluation on Human 3.6M}

We now turn to the Human36M dataset, where we first evaluate the different components of our approach, and then compare to the state-of-the-art multi-view methods.
Note that here we consider a setting where no additional data is used to train our models.  We report the results in Table \ref{tab:h36}.
Considering the ablation study, we obtain results that are consistent with what we observed on the TotalCapture dataset: performing simple feature fusion (\textit{Fusion}) yields a  $18 \%$ improvement over the monocular baseline. A further $\sim10 \%$ improvement can be reached by using \textit{Canonical Fusion (no DLT)}.
Finally, training our architecture by back-propagating through the triangulation layer (\textit{Canonical Fusion}) allows to further improve our accuracy by  $7 \%$. We show some uncurated test samples from our model in Figure \ref{fig:vis}(b).

We then compare our model to the state-of-the-art methods. Here we can compare our method to the one of \cite{Qiu19} just by comparing fusion techniques  (see \textit{Canonical Fusion (no DLT)} vs Qui \etal \cite{Qiu19} (no RPSM) in Table \ref{tab:h36}). We see that our methods outperform theirs by $\sim15\%$, which is significant and indicates the superiority of our fusion technique. 
Similar to what observed in Section \ref{sec:ablate}, our best performing method is even superior to the off-line volumetric of \cite{Qiu19}, which uses a strong bone-length prior (Qui \etal \cite{Qiu19} Fusion + RPSM). Our method outperforms all other multi-view approaches by a large margin. Note that in this setting we cannot compare to \cite{Iskakov19}, as they do not report results without using additional data. 
\subsection{Exploiting additional data}
\label{sec:eval_add}
\begin{table}[h]
	\vspace{-8pt}
	\begin{center}
		\scalebox{0.8}{
		\begin{tabular}{c|c|c|c}
			\hline
			Methods & Model size  & Inference Time & MPJPE\\
			\hline
			Qui \etal \cite{Qiu19} Fusion + RPSM &  2.1GB & 8.4s & 26.2 \\ 
			Iskakov \etal \cite{Iskakov19} Algebraic & 320MB & 2.00s & 22.6 \\
			Iskakov \etal \cite{Iskakov19} Volumetric &  643MB & 2.30s & \textbf{20.8}\\
			\hline
			Ours, Baseline & \textbf{244MB} &  \textbf{0.04s} & 34.2 \\
			Ours, Canonical Fusion & 251MB & \textbf{0.04s} & 21.0 \\
			\hline
		\end{tabular}}
	\end{center}
	\caption{Additional training data setup. We compare our method to the state-of-the-art approaches in terms of performance, inference time, and model size on the Human3.6M dataset.}
	\label{tab:bench}
\end{table}

To compare to the concurrent model in \cite{Iskakov19}, we consider a setting in which we exploit additional training data. We adopt the same pre-training strategy as \cite{Iskakov19}, that is we pre-train a monocular pose estimation network on the COCO dataset \cite{Lin14}, and fine-tune jointly on Human3.6M and MPII \cite{Andriluka14} datasets. We then simply use these pre-trained weights to initialize our network. We also report results for \cite{Qiu19}, which trains its detector jointly on MPII and Human3.6M. The results are reported in Table \ref{tab:bench}.

First of all, we observe that \textit{Canonical Fusion} outperforms our monocular baseline by a large margin ($\sim 39 \%$). Similar to what was remarked in the previous section, our method also outperforms \cite{Qiu19}. The gap, however, is somewhat larger in this case ($\sim 20 \%$). 
Our approach also outperforms the triangulation baseline of  (Iskakov \etal \cite{Iskakov19} Algebraic), indicating that our fusion technique if effective in reasoning about multi-view input images. Finally, we observe that our method reaches accuracy comparable to the volumetric approach of (Iskakov \etal \cite{Iskakov19} Volumetric).

To give insight on the computational efficiency of our method, in Table \ref{tab:bench} we report the size of the trained models in memory, and also measure their inference time (we consider a set of 4 images and measure the time of a forward pass on a Pascal TITAN X GPU and report the average over 100 forward passes). Comparing model size, \textit{Canonical Fusion} is much smaller than other models and introduces only a negligible computational overhead compared to our monocular \textit{Baseline}. Comparing the inference time, both our models yield a real-time performance ($\sim 25 fps$) in their un-optimized version, which is much faster than other methods. In particular, it is about 50 times faster than (Iskakov \etal \cite{Iskakov19} Algebraic) due to our efficient implementation of DLT and about 57 times faster than (Iskakov \etal \cite{Iskakov19} Volumetric) due to using DLT plus 2D CNNs instead of a 3D volumetric approach.

\vspace{-2pt}
\section{Conclusions}
\vspace{-4pt}
We propose a new multi-view fusion technique for 3D pose estimation that is capable of reasoning across multi-view geometry effectively, while introducing negligible computational overhead with respect to monocular methods.
Combined with our novel formulation of DLT transformation, this results in a real-time approach to 3D pose estimation from multiple cameras.
We report the state-of-the-art performance on standard benchmarks when using no additional data, flexibility to unseen camera settings, and accuracy comparable to far-more computationally intensive volumetric methods when allowing for additional 2D annotations.

\section{Acknowledgments}
We would like to thank Giacomo Garegnani for the numerous and insightful discussions on singular value decomposition. This work was completed during an internship at Facebook Reality Labs, and supported in part by the Swiss National Science Foundation.


{\small
\bibliographystyle{ieee_fullname}
\bibliography{egbib} 
}

\section{Supplementary Material}
\subsection{Architectures}

\begin{figure*}[t]
	\centering
	\begin{minipage}[h]{1.0\textwidth}
		\begin{overpic}[clip, trim=1.5cm 0cm 1.5cm 0cm,width=1.0\textwidth]{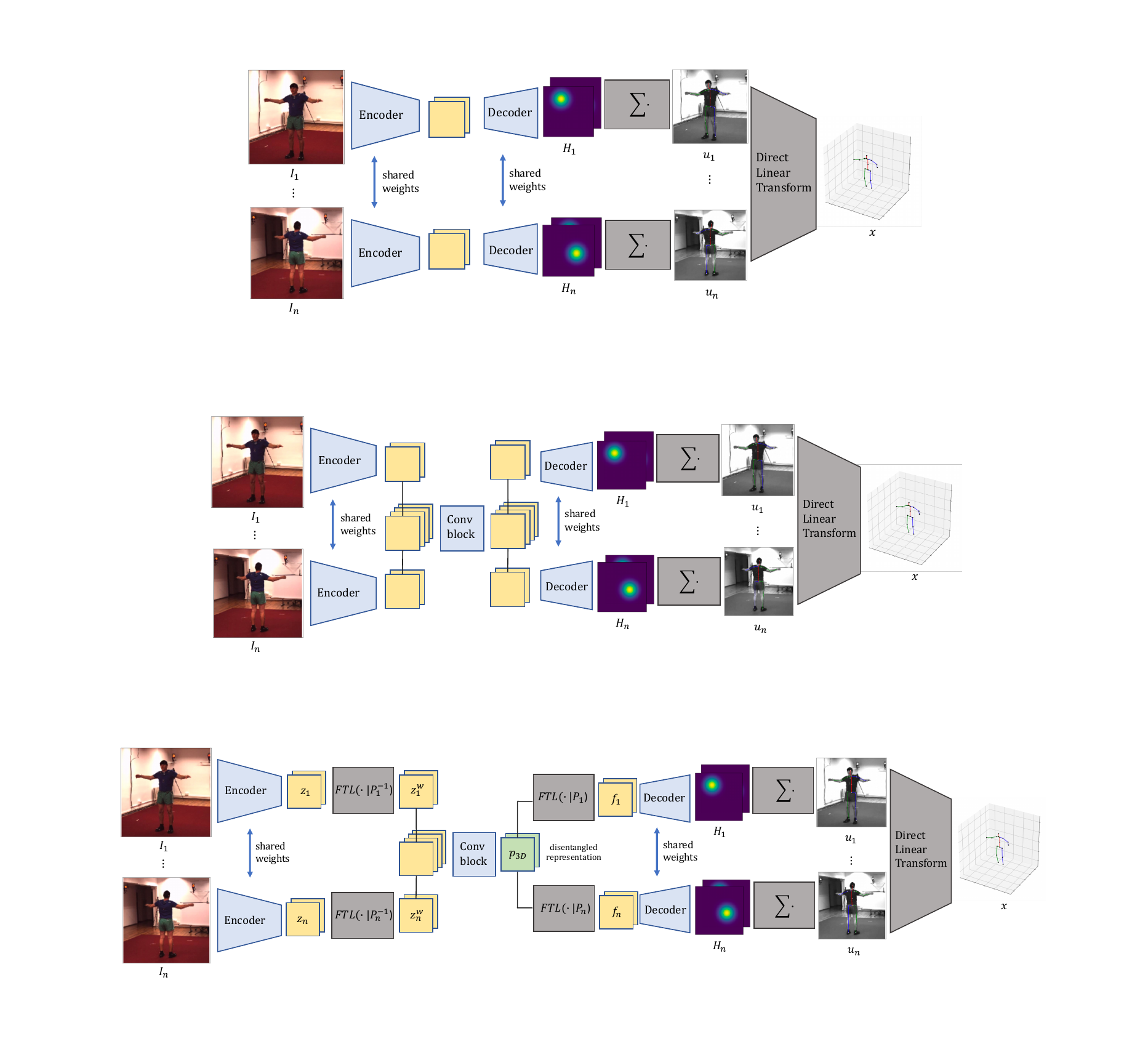}
			\put(42,68){a) \large{\textbf{Baseline}}}
			\put(42,35){b) \large{\textbf{Fusion}}}
			\put(38,2){c) \large{\textbf{Canonical Fusion}}}
		\end{overpic}
	\end{minipage}
	\vspace{1.5cm}
	\caption{Overview of different multi-view architectures: a) \textit{baseline}, which detects 2D locations of joints for each view separately and then lifts detections to 3D via DLT triangulation. b) the multi-view feature fusion technique (\textit{fusion}) that performs joint reasoning in the latent space, similar in spirit to the methods of \cite{Kadkhodamohammadi18, Qiu19}. This approach does not exploit epipolar geometry and hence overfits to the camera setting. c) our novel fusion method (\textit{canonical fusion}), exploiting camera transform layers to fuse views flexibly into a unified pose representation that is disentangled from camera view-points and thus can generalize to novel views.
	}
	\label{fig:archis}
\end{figure*}
In Figure \ref{fig:archis}, we depict the different architectures (\textit{baseline}, \textit{fusion}, \textit{canonical fusion}) compared in the main article. 
Recall that our encoder consists of a ResNet152 \cite{He16} backbone pre-trained on ImageNet \cite{Deng09} for all three architectures, taking in $256\times 256$ image crops as input and producing $2048\times 18 \times 18$ features maps. 
Similarly, all methods share the same convolutional decoder, consisting of
\begin{itemize}
	\item ConvTranspose2D(2048, 256) + BatchNorm + ReLU
	\item ConvTranspose2D(256, 256) + BatchNorm + ReLU
	\item ConvTranspose2D(256, 256) + BatchNorm + ReLU
	\item Conv2D(256, K).
\end{itemize} 
This produces $K\times 64 \times 64$ output heatmaps, where $K$  is the number of joints. The only difference between the networks is in the feature fusion module, respectively defined as follows:
\begin{itemize}
	\item \textit{baseline}: no feature fusion.
	\item \textit{fusion}: a $1 \times 1$ convolution is first applied to map features from $2048$ channels to $300$. Then, the feature maps from different views are concatenated to make a feature map of size $n\times 300$, where $n$ indicates the number of views. This feature map is then processed jointly by two $1 \times 1$ convolutional layers, finally producing a feature map with $n\times 300$ channels, which is later split into view-specific feature maps with $300$ channels in each view. Each view-specific feature map is then lifted back to $2048$ channels.
	\item \textit{canonical fusion}: a $1 \times 1$ convolution is first applied  to map features from $2048$ channels to $300$. The feature maps from different views are then transformed to a shared canonical representation (world coordinate system) by feature transform layers.  Once they live in the same coordinate system, they are concatenated into a $n\times 300$ feature map and processed jointly by two $1 \times 1$ convolutional layers, producing a \textit{unified} feature map with $300$ channels that is disentangled from the camera view-point. This feature map, denoted as $p_{\text{3D}}$ in the main article, is then projected back to each view-point by using feature transform layers and the corresponding camera transform matrix. Finally each view-specific feature map is mapped back to $2048$ channels. Note that in contrast to \textit{fusion} that learns separate latent representations for different views, in \textit{canonical fusion} all views are reconstructed from the same latent representation, effectively forcing the model to learn a unified representation across all views.
\end{itemize}

\subsection{Efficient Direct Linear Transformation}
\label{sec:DLT}

In this section we prove Theorem \ref{thm:sing} from the main article, and then illustrate how in practice we use it to design an efficient algorithm for Direct Linear Transformation by using Shifted Inverse Iterations method \cite{Quarteroni10}. Finally, we provide some insight on why SVD is not efficient on GPUs (see Figure 3d in the main article).

\begin{theorem}
	\label{thm:sing}
	Let $A$ be the DLT matrix associated with the non-perturbed case, i.e. $\sigma_{\text{min}}(A)  = 0$. 
	Let us assume i.i.d Gaussian noise $\varepsilon =(\varepsilon_u,\varepsilon_v)\sim \mathcal N(0,s^2I)$ in our 2d observations, i.e. $(u^*, v^*) = (u + \varepsilon _u, v + \varepsilon _v)$, and let us denote $A^*$ the DLT matrix associated with the perturbed system. 
	Then, it follows that:
	\begin{align}
	\label{eq:thm}
	0 \leq \mathbb E [\sigma_{\text{min}}(A^*)] \leq  C s, \, \, \text{where } C=C(\{u_i, P_i \}_{i=1}^N)
	\end{align}
\end{theorem}

	\begin{proof}
	Let us recall the structure of matrix $A \in \mathbb R ^ {2n \times 4}$, which is the DLT matrix for non-noisy 2D observations:
	\begin{align}
	A = \begin{bmatrix} 
	\vdots  \\
	u_i p_i ^ {3T} - p_i ^ {1T} \\
	v_i p_i ^ {3T} - p_i ^ {2T}  \\
	\vdots  \\
	\end{bmatrix}.
	\end{align}
	
Now considering noisy observations $(u^* _ i, v^* _ i) = (u _ i + \varepsilon _{2i}, v _ i + \varepsilon _{2i+1})$, where we drop the subscripts $u,v$ from $\varepsilon$ (as noise is i.i.d.),  the DLT matrix can be written as

\begin{align}
A^* &= \begin{bmatrix} 
\vdots  \\
(u_i +  \varepsilon _{2i}) \, p_i ^ {3T} - p_i ^ {1T} \\
(v_i +  \varepsilon _{2i+1}) \, p_i ^ {3T} - p_i ^ {2T}  \\
\vdots  \\
\end{bmatrix},
\end{align}
which is equivalent to
\begin{align}
A^* &= A + \begin{bmatrix} 
\vdots  \\
  \varepsilon _{2i} \, p_i ^ {3T}  \\
 \varepsilon _{2i+1} \, p_i ^ {3T}  \\
\vdots  \\
\end{bmatrix} \\ 
 &= A + \begin{bmatrix} 
 \ddots  \\
&  \varepsilon _{2i} & &    \\
 & & \varepsilon _{2i+1}  & \\
 & & & \ddots  \\
 \end{bmatrix}
 \begin{bmatrix} 
 \vdots  \\
 p_i ^ {3T}    \\
 p_i ^ {3T}  \\
 \vdots  \\
 \end{bmatrix} \\
 &= A + \Sigma P, \label{eq:def}
\end{align}
where $\Sigma \in \mathbb R ^ {2n \times 2n}$ and $P \in \mathbb R ^ {2n \times 4}$.

Using the classical perturbation theory (see Stewart \etal \cite{Stewart98} for an overview), we can write
\begin{align}
| \sigma_{\text{min}}(A^*) - \sigma_{\text{min}}(A) | \leq \| A^* - A \|_2.
\end{align}

By exploiting $\sigma_{\text{min}}(A) = 0$, Equation \ref{eq:def}, and the fact that singular values are always positive we can infer
\begin{align}
 \sigma_{\text{min}}(A^*)  \leq \| \Sigma P \|_2.
\end{align}

Then by leveraging Cauchy-Schwartz inequality \cite{Cauchy21} and recalling that the norm 2 of a diagonal matrix is bounded by the absolute value of the biggest element in the diagonal we get
\begin{align}
\sigma_{\text{min}}(A^*)  \leq \| \Sigma \|_2 \| P \|_2 \leq \| P \|_2  \max_i | \varepsilon _{i}|.
\end{align}

Recall that that the max of $2n$ i.i.d. variables is smaller than their sum, so we can write
\begin{align}
\sigma_{\text{min}}(A^*)  \leq \| P \|_2  \sum _ {i=0} ^{2n-1}  |\varepsilon _{i}|. \label{eq:almost}
\end{align}

We can then simply take the expected value on both sides  of Equation (\ref{eq:almost}) and obtain
\begin{align}
\mathbb E \Big[ \sigma_{\text{min}}(A^*) \Big]  &\leq \mathbb E \Big[ \| P \|_2   \sum _ {i=0} ^{2n-1}  |\varepsilon _{i}|\Big] \\
&\leq \| P \|_2   \sum _ {i=0} ^{2n-1} \mathbb E [|\varepsilon _{i}|] \\
&\leq \| P \|_2 \, 2n \, \mathbb E [|\varepsilon _{0}|].
\end{align}

Knowing that the expected value of the \textit{half-normal} distribution is $ E [|\varepsilon _{i}|]  = s \sqrt{2/\pi}$  we finally obtain
\begin{align}
\mathbb E [\sigma_{\text{min}}(A^*)] \leq 2n \sqrt{2/\pi} \| P \|_2 \, s = C s.
\end{align}

The other side of inequality (\ref{eq:thm}) trivially follows from the fact that singular values are always positive. 
\end{proof}

In the main article, we proposed (in Algorithm 1) to find the singular vector of $A^*$ associated with $\sigma_{\text{min}}(A^*)$ by means of Shifted Inverse Iterations (SII) \cite{Quarteroni10} applied to matrix $A^{*T}A^*$. 
This iterative algorithm (which takes as input a singular value estimate $\mu$) has the following properties:
\begin{enumerate}
	\item The iterations will converge to the eigenvector that is closest to the provided estimate.
	\item The rate of convergence of the algorithm is geometric, with ratio $ \dfrac{\sigma_{\text{4}}(A^*)+\mu }{\sigma_{\text{3}}(A^*) + \mu }$, where $\sigma_{3} \geq \sigma_{4} = \sigma_{\text{min}}$.
\end{enumerate}

Combining property 1 with the result of Theorem \ref{thm:sing} ascertains that Algorithm 1 will converge to the desired singular vector if we provide it with a small value for $\mu$. Although in theory we could set $\mu=0$, in practice we choose $\mu = 0.001$ to avoid numerical instabilities when matrix $A^{*T}A^*$ is close to being singular.

Note also that property 2 is confirmed by what we see in Figure 3b in the main article, where the number of iterations needed by the algorithm to reach convergence increases with more Gaussian noise in the 2D observation. In practice, we have found two iterations to be sufficient in our experiments.
\\
\\
\textbf{SVD parallelization on GPU.} In our experiments, carried in \texttt{PyTorch v1.3} on a Pascal TITAN X GPU, we found DLT implementations based on Singular Value Decomposition (SVD) to be inefficient on GPU (see Figure 3d in the main paper). Below we provide an insight on why this is the case.

SVD numerical implementations \cite{Dongarra14} involve two steps:
\begin{enumerate}
	\item Two orthogonal matrices Q and P are applied to the left and right of matrix $A$, respectively, to reduce it to a bidiagonal
	form, $B = Q^TAP$.
	\item Divide and conquer or QR iteration is then used to
find both singular values and left-right singular vectors of $B$ yielding
$ B = \bar U^T \Sigma \bar V$. Then, singular vectors of $B$
are back-transformed to singular vectors of $A$ by $U = Q \bar U$ and $V = \bar V P$.
\end{enumerate} 

There are many ways to formulate these problems mathematically and solve them numerically, but in all cases, designing an efficient computation is challenging because
of the nature of the reduction algorithm. In particular, the orthogonal transformations applied
to the matrix are two-sided, i.e., transformations are applied on both the left and
the right side of the matrix. This creates data dependencies and prevents the use of standard techniques to increase the computational efficiency of the operation, for example
blocking and look-ahead, which are used extensively in the one-sided algorithms (such as in LU, QR, and
Cholesky factorizations \cite{Dongarra14}). A recent work \cite{Wang19} has looked into ways to increase stability of SVD while reducing its computational time.  Similarly, we also found SVD factorization to be slow, which motivated us to design  
a more efficient solution involving only GPU-friendly operations (see Algorithm 1 in the main article).  

\subsection{Feature Transform Layer}
\begin{figure*}[t]
	\centering
	\begin{overpic}[clip, trim=0cm 0cm 0cm 0cm,width=1.0\textwidth]{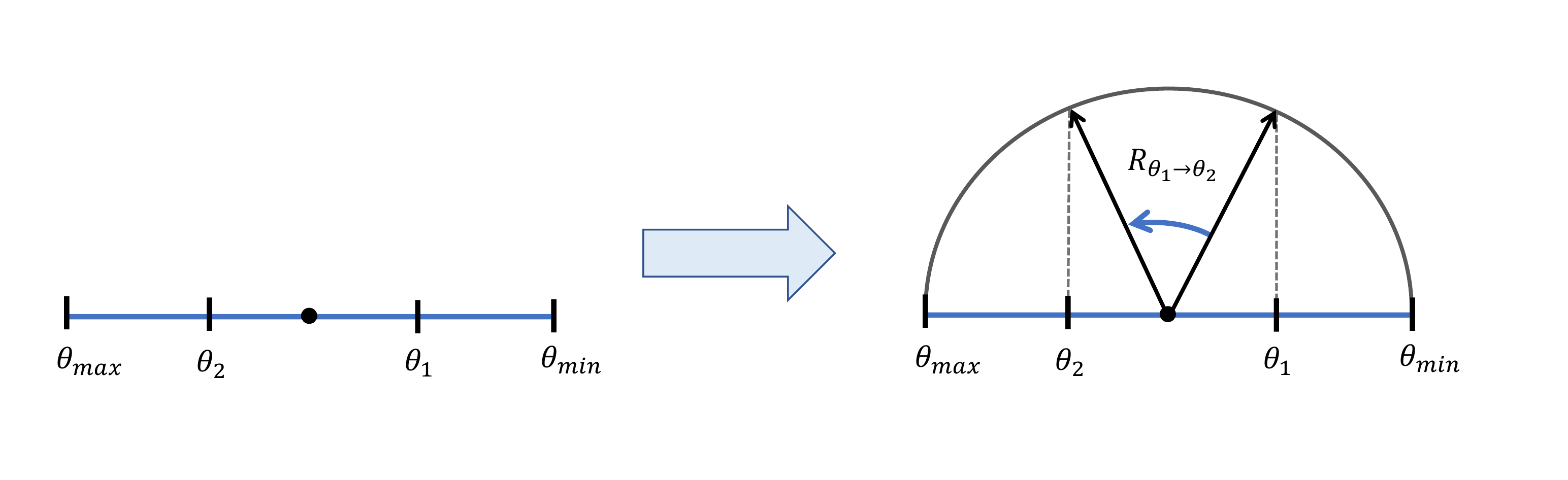}
		\put (35,26) {$ \cos \theta = \dfrac{ \theta - 0.5*(\theta_{\text{max}}+ \theta_{\text{min}})}{ 0.5*(\theta_{\text{max}}- \theta_{\text{min}})}$}
		\put (35,20) {$ \sin \theta = \sqrt{1 -  \cos ^ 2 \theta }$}
		
		\put (77,27) {$ R_ {\theta} = \begin{bmatrix}   \cos \theta &  \sin \theta \\
		- \sin\theta & \cos \theta \\
		\end{bmatrix}$}
		
	\end{overpic}
	\caption{FTL encodes transformations by mapping them onto circles in the feature space. Consider the setting in which a factor of variation $\theta$ (e.g. $x$-component of camera position in world coordinates), defined in the interval $\theta \in [\theta_{\text{min}}, \theta_{\text{max}}]$, changes from $\theta=\theta_1$ to $\theta=\theta_2$.
	Exploiting trigonometry, we can map this transformation onto a circle, as depicted on the right-hand side of the figure, where the transformation is defined as a rotation. 
	}
	\label{fig:ftl}
\end{figure*}
Below we first review feature transform layers (FTLs), introduced in \cite{Worrall17} as an effective way to learn interpretable embeddings. Then we explain how FTLs are used in our approach.

Let us consider a representation learning task, where images $\mathbf X$ and $\mathbf Y$ are related by a known transform $T$ and the latent vector $\mathbf x$ is obtained from $\mathbf X$ via an encoder network. The feature transform layer performs a linear transformation on $\mathbf x$ via transformation matrix $F_T$ such that the output of the layer is defined as
\begin{align}
\mathbf y = F_T [\mathbf x]   = F_T \mathbf x,
\end{align}
where  $\mathbf y$ is the transformed representation. Finally $\mathbf y$ is decoded to reconstruct the target sample $\mathbf Y$.
This operation forces the neural network to learn a mapping from image-space to feature-space while preserving the intrinsic structure of the transformation. 

In practice, the transforming matrix $F_T$ should be chosen such that it is invertible and norm preserving. To this end \cite{Worrall17} proposes to use rotations since they are simple and respect these properties. Periodical transformations can trivially be converted to rotations. Although less intuitive, arbitrary transformation defined on an interval can also be thought of as rotations by mapping them onto circles in feature space. Figure \ref{fig:ftl} illustrates in detail how to compute this mapping.

Note that if $\mathbf X$ and $\mathbf Y$ differ by more than one factor of variation, disentanglement can be achieved by transforming features as follows:
\begin{align}
\mathbf y = F_{T_1, \dots, T_n} [\mathbf x]   = \begin{bmatrix} 
 F_{T_1} \\
& \ddots &    \\
& &  F_{T_n}  \\
\end{bmatrix} \mathbf x .
\end{align}

In \cite{Worrall17} FTLs are presented as a way to learn representations from data that are 1) interpretable, 2) disentangled, and 3) better suited for down-stream tasks, such as classification.

In our work, we use FTLs to feed camera transformations explicitly into the network in order to design an architecture that can reason both efficiently and effectively about epipolar geometry in the latent space. As a consequence, the model learns a camera-disentangled representation of 3D pose, that recovers 2D joint locations from multi-view input imagery. This shows that FTLs can be used to learn disentangled latent representations also in supervised learning tasks.

\subsection{Additional results}
In Figures \ref{fig:h36_res}  and \ref{fig:tc_res} we provide additional visualizations, respectively for TotalCapture (using both seen and unseen cameras) and Human3.6M datasets. These uncurated figures illustrate the quality of our predictions. We encourage the reader to look at our supplementary videos for further qualitative results. 

\begin{figure*}[t]
	\centering
	\vspace{-8pt}
	\begin{overpic}[clip, trim=0cm 0cm 0cm 0cm,width=1.0\textwidth]{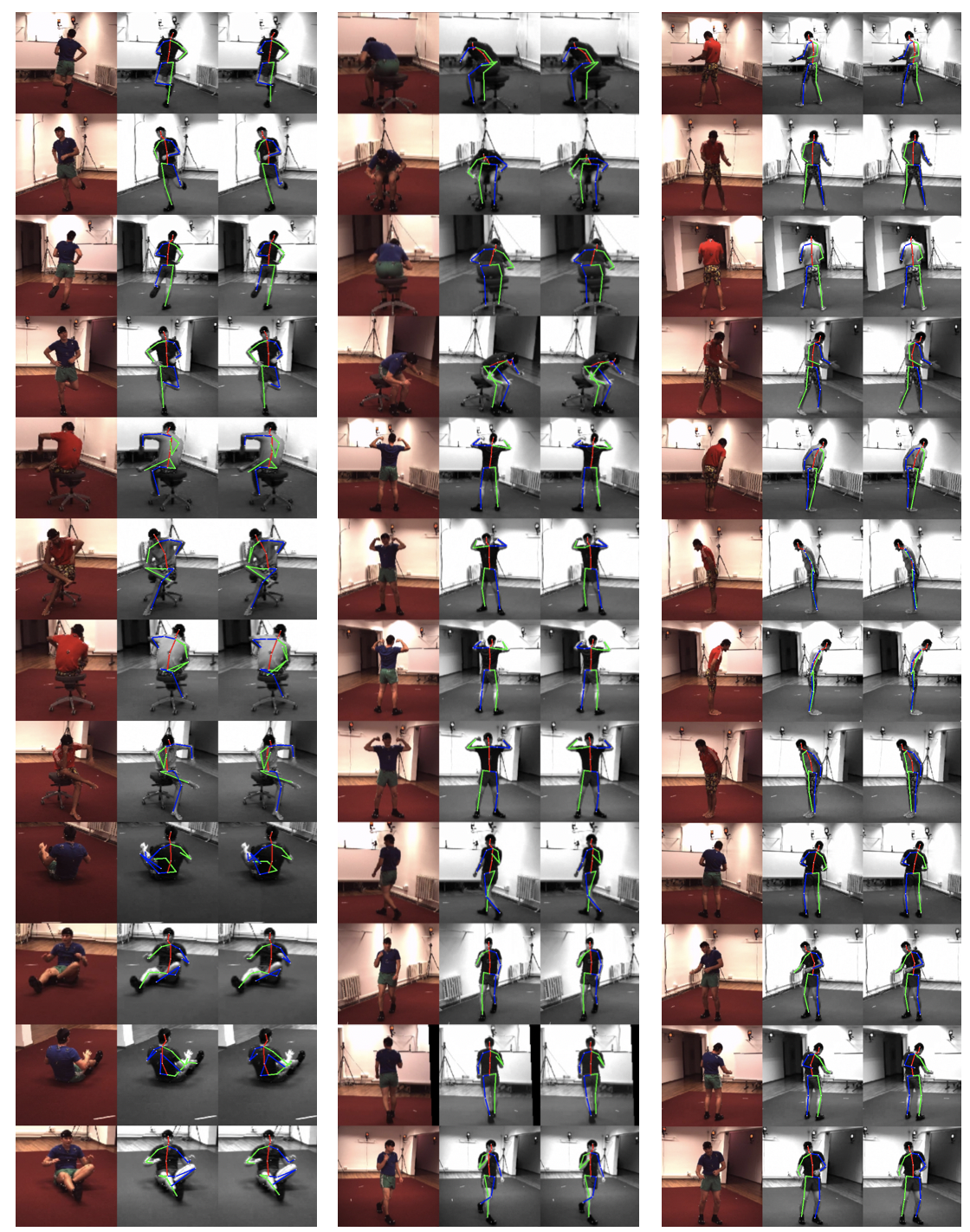}
		\put (4,100) {\small{input}}
		\put (10,100) {\small{ground truth}}
		\put (19,100) {\small{prediction}}
		
	\put (30,100) {\small{input}}
	\put (36,100) {\small{ground truth}}
	\put (45,100) {\small{prediction}}
		
	\put (56,100) {\small{input}}
	\put (62,100) {\small{ground truth}}
	\put (71,100) {\small{prediction}}
	
	\put (0,95) {\small{1}}
	\put (0,86.5) {\small{2}}
	\put (0,78) {\small{3}}
	\put (0,70.0) {\small{4}}

	\put (0,62) {\small{1}}
	\put (0,53) {\small{2}}
	\put (0,45.5) {\small{3}}
	\put (0,37.0) {\small{4}}

	\put (0,29) {\small{1}}
	\put (0,20.5) {\small{2}}
	\put (0,12.5) {\small{3}}
	\put (0,4) {\small{4}}

	\put (26.2,95) {\small{1}}
	\put (26.2,86.5) {\small{2}}
	\put (26.2,78) {\small{3}}
	\put (26.2,70.0) {\small{4}}

	\put (26.2,62) {\small{1}}
	\put (26.2,53) {\small{2}}
	\put (26.2,45.5) {\small{3}}
	\put (26.2,37.0) {\small{4}}

	\put (26.2,29) {\small{1}}
	\put (26.2,20.5) {\small{2}}
	\put (26.2,12.5) {\small{3}}
	\put (26.2,4) {\small{4}}
	
	\put (52.5,95) {\small{1}}
	\put (52.5,86.5) {\small{2}}
	\put (52.5,78) {\small{3}}
	\put (52.5,70.0) {\small{4}}

	\put (52.5,62) {\small{1}}
	\put (52.5,53) {\small{2}}
	\put (52.5,45.5) {\small{3}}
	\put (52.5,37.0) {\small{4}}

	\put (52.5,29) {\small{1}}
	\put (52.5,20.5) {\small{2}}
	\put (52.5,12.5) {\small{3}}
	\put (52.5,4) {\small{4}}
		
	\end{overpic}
	\caption{Randomly picked samples from the test set of Human3.6M. Numbers denote cameras.
		To stress that the pose representation learned by our network is effectively \textit{disentangled} from the camera view-point, we intentionally show predictions \textit{before} triangulating them, rather than re-projecting triangulated keypoints to the image space. 
	}
	\label{fig:h36_res}
\end{figure*}

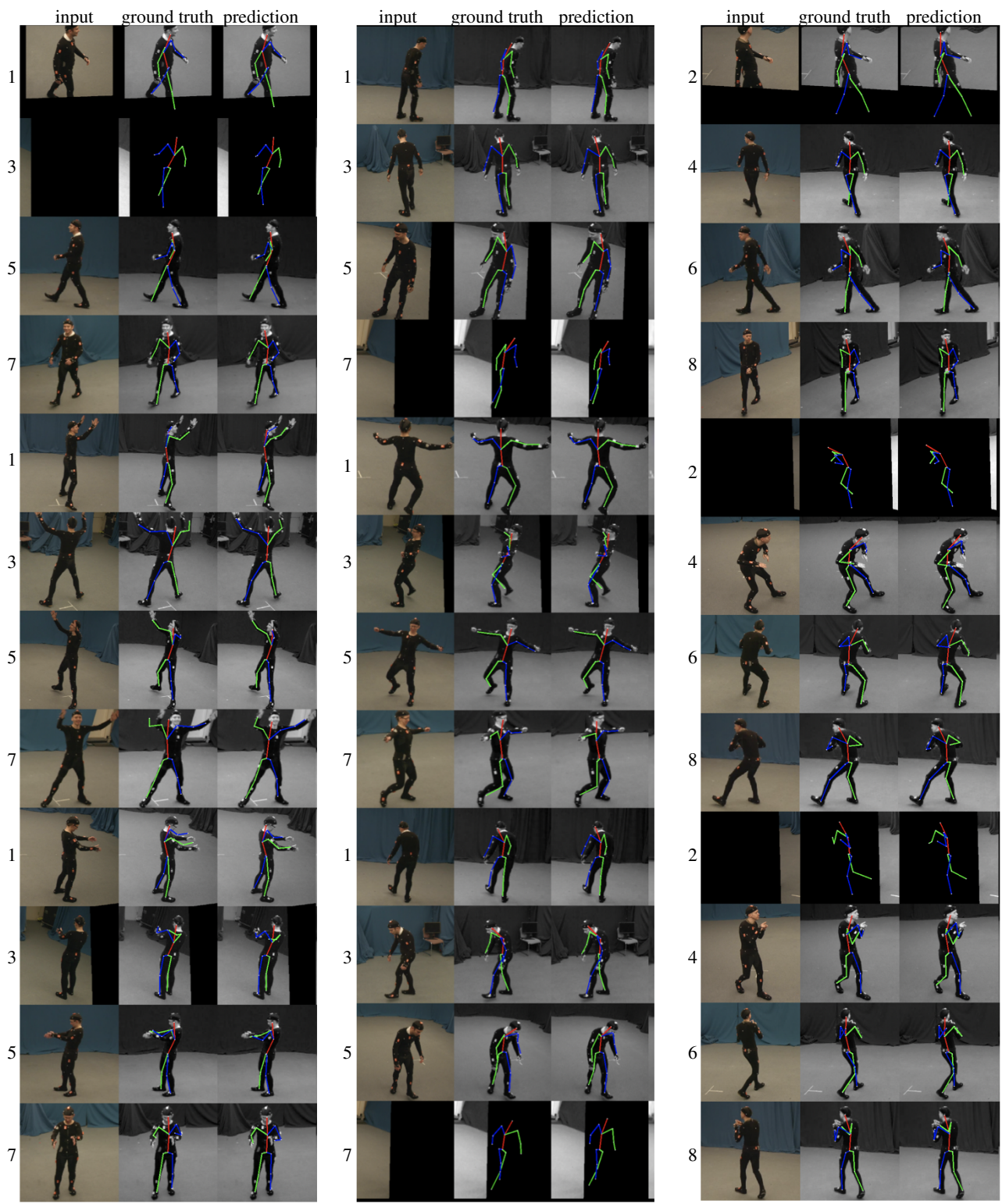
\begin{figure*}[t]
	\centering
	\begin{overpic}[clip, trim=0cm 0cm 0cm 0cm,width=1.0\textwidth]{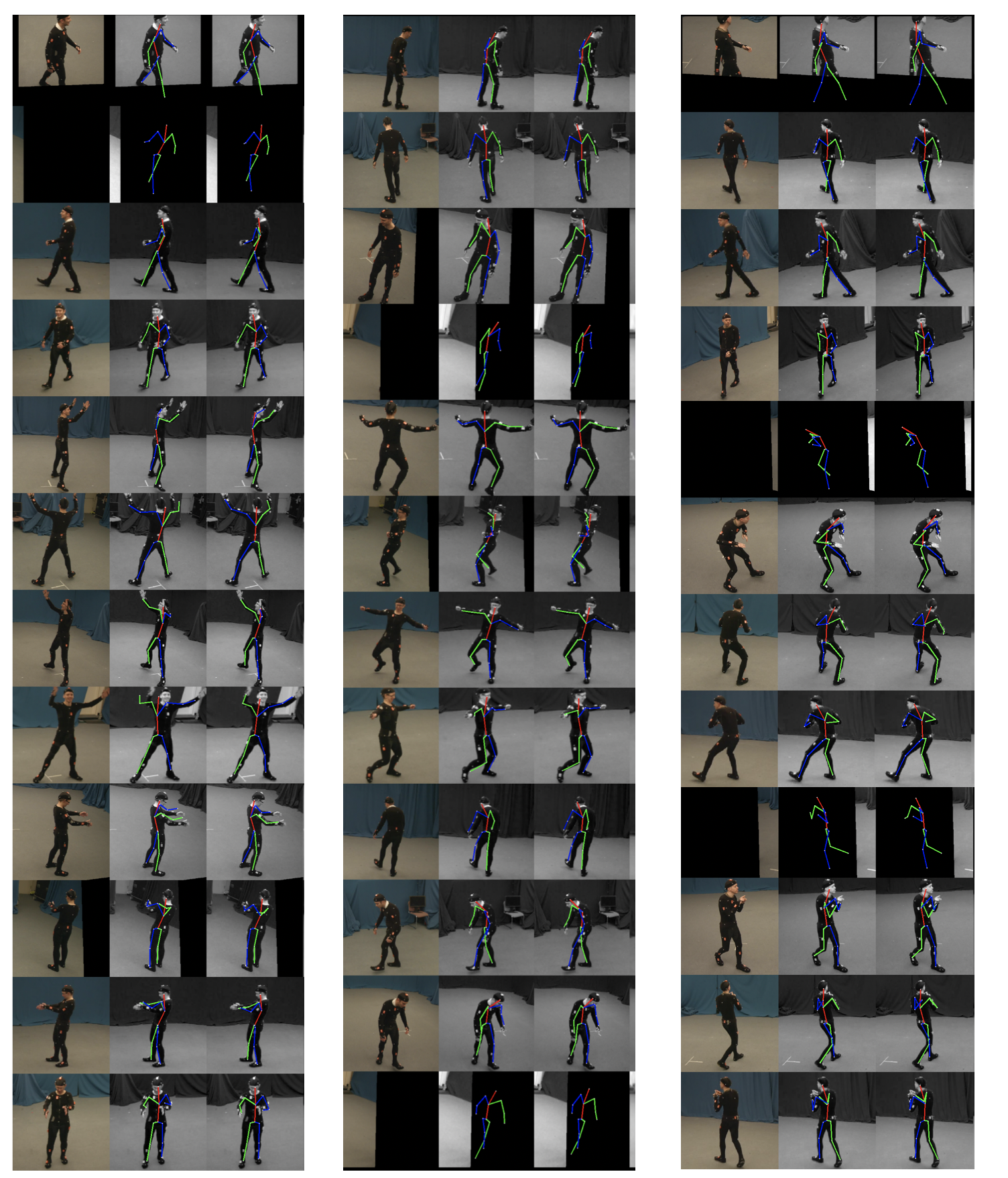}
		
		\put (21,103) {\large{{\textbf{seen camera-setting}}}}
		\put (62,103) {\large{{\textbf{new camera-setting}}}}
		
		\put (4,99) {\small{input}}
		\put (9.5,99) {\small{ground truth}}
		\put (18,99) {\small{prediction}}
		
		\put (31,99) {\small{input}}
		\put (37,99) {\small{ground truth}}
		\put (46,99) {\small{prediction}}
		
		\put (60,99) {\small{input}}
		\put (66,99) {\small{ground truth}}
		\put (75,99) {\small{prediction}}
		
		\put (0,94) {\small{1}}
		\put (0,86.5) {\small{3}}
		\put (0,78) {\small{5}}
		\put (0,70.0) {\small{7}}

		\put (0,62) {\small{1}}
		\put (0,53.5) {\small{3}}
		\put (0,45.5) {\small{5}}
		\put (0,37.0) {\small{7}}

		\put (0,29) {\small{1}}
		\put (0,20.5) {\small{3}}
		\put (0,12.5) {\small{5}}
		\put (0,4) {\small{7}}
		
		\put (28,94) {\small{1}}
		\put (28,86.5) {\small{3}}
		\put (28,78) {\small{5}}
		\put (28,70.0) {\small{7}}

		\put (28,61.5) {\small{1}}
		\put (28,53.5) {\small{3}}
		\put (28,45.5) {\small{5}}
		\put (28,37.0) {\small{7}}

		\put (28,29) {\small{1}}
		\put (28,20.5) {\small{3}}
		\put (28,12.5) {\small{5}}
		\put (28,4) {\small{7}}
		
		\put (56.8,94) {\small{2}}
		\put (56.8,86.5) {\small{4}}
		\put (56.8,78) {\small{6}}
		\put (56.8,70.0) {\small{8}}

		\put (56.8,61) {\small{2}}
		\put (56.8,53.5) {\small{4}}
		\put (56.8,45.5) {\small{6}}
		\put (56.8,37.0) {\small{8}}

		\put (56.8,29) {\small{2}}
		\put (56.8,20.5) {\small{4}}
		\put (56.8,12.5) {\small{6}}
		\put (56.8,4) {\small{8}}

	\end{overpic}
	\caption{Randomly picked samples from the test set of TotalCapture.
		Numbers denote cameras.
		In the two left columns we test our model on unseen images captured from seen camera view-points. In the right column, instead, we use images captured from unseen camera view-points.
		To stress that the pose representation learned by our network is effectively \textit{disentangled} from the camera view-point, we intentionally show predictions \textit{before} triangulating them, rather than re-projecting triangulated keypoints to the image space. 
	}
	\label{fig:tc_res}
\end{figure*}

\end{document}